\pdfoutput=1
\documentclass[10pt,twocolumn]{article}

\usepackage[T1]{fontenc}
\usepackage[utf8]{inputenc}
\usepackage{lmodern}
\usepackage[protrusion=true,expansion=false]{microtype}

\usepackage[letterpaper,margin=0.75in]{geometry}
\usepackage{parskip}
\usepackage{amsmath}
\usepackage{amssymb}
\usepackage{amsthm}
\usepackage{mathtools}

\usepackage{graphicx}
\usepackage{booktabs}
\usepackage{array}
\usepackage{tabularx}
\usepackage{float}
\usepackage{caption}
\usepackage[table]{xcolor}

\usepackage{listings}
\lstdefinestyle{faop}{
  basicstyle=\ttfamily\small,
  breaklines=true,
  columns=fullflexible,
  keepspaces=true,
  showstringspaces=false,
  frame=single,
  framesep=4pt,
  rulecolor=\color{gray!50},
  backgroundcolor=\color{gray!6},
  xleftmargin=2pt,
  xrightmargin=2pt,
  aboveskip=6pt,
  belowskip=6pt,
}
\usepackage{titlesec}
\titlespacing*{\section}{0pt}{1.1ex plus 0.3ex minus 0.2ex}{0.5ex plus 0.1ex}
\titlespacing*{\subsection}{0pt}{0.8ex plus 0.2ex minus 0.1ex}{0.3ex plus 0.1ex}
\titlespacing*{\subsubsection}{0pt}{0.6ex plus 0.2ex minus 0.1ex}{0.2ex plus 0.1ex}
\usepackage[colorlinks=true,linkcolor=blue!60!black,citecolor=blue!60!black,urlcolor=blue!60!black]{hyperref}

\usepackage[numbers,sort&compress]{natbib}
\newtheorem{theorem}{Theorem}
\newtheorem{lemma}{Lemma}
\newtheorem{proposition}{Proposition}
\newtheorem{definition}{Definition}
\newtheorem{convention}{Convention}

\newcommand{\Sin}{S_{\mathit{in}}}
\newcommand{\Sout}{S_{\mathit{out}}}
\newcommand{\Sopt}{S_{\mathit{opt}}}
\newcommand{\seqop}{\mathit{seq}}
\newcommand{\parop}{\mathit{par}}
\newcommand{\ifop}{\mathit{if}}
\newcommand{\loopop}{\mathit{loop}}
\newcommand{\callop}{\mathit{call}}
\newcommand{\Det}{\operatorname{Det}}
\newcommand{\Comp}{\operatorname{Comp}}
\newcommand{\leavesop}{\operatorname{leaves}}
\newcommand{\predsop}{\operatorname{preds}}

\newcommand{\nopsym}{\mathrm{nop}}
\newcommand{\opHTTP}{\mathrm{op}_{\mathrm{HTTP}}}
\newcommand{\opDB}{\mathrm{op}_{\mathrm{DB}}}
\newcommand{\opLLM}{\mathrm{op}_{\mathrm{LLM}}}
\newcommand{\opTool}{\mathrm{op}_{\mathrm{tool}}}
\newcommand{\opFile}{\mathrm{op}_{\mathrm{file}}}
\newcommand{\opQueue}{\mathrm{op}_{\mathrm{queue}}}
\newcommand{\opRet}{\mathrm{op}_{\mathrm{ret}}}
\newcommand{\opRag}{\mathrm{op}_{\mathrm{rag}}}
\newcommand{\opPlan}{\mathrm{op}_{\mathrm{plan}}}
\newcommand{\opCompute}{\mathrm{op}_{\mathrm{compute}}}
\newcommand{\sem}[1]{\mathopen{[\![}#1\mathclose{]\!]}}

\newcommand{\keybox}[1]{%
  \par\smallskip\noindent
  \setlength{\fboxsep}{7pt}%
  \colorbox{blue!7}{\parbox{\dimexpr\linewidth-2\fboxsep\relax}{#1}}%
  \par\smallskip}

\usepackage{fancyhdr}
\newcommand{\CopyrightLine}{}
\fancypagestyle{plain}{%
  \fancyhf{}%
  \fancyfoot[L]{\footnotesize\color{gray!60}\CopyrightLine}%
  \fancyfoot[R]{\footnotesize\color{gray!60}\thepage}%
}

\title{\textbf{\Large FAOPs: Normal Forms and an Algebra for AI Agent Workflows}}
\author{%
{\large Ali M. Shanghooshabad}\\[0.35em]
\normalsize JPMorgan Chase \& Co.\\[0.35em]
\normalsize\texttt{shanghoosh@gmail.com}\quad\texttt{ali.shanghooshabad@jpmorgan.com}%
}
\date{}

\begin{document}

\makeatletter
\twocolumn[
\begin{@twocolumnfalse}
\maketitle

\begin{abstract}

Organizations increasingly encode operational procedures as multi-step agent \emph{skills}, here called Agent Operating Procedures (AOPs), that must be traceable, auditable, reproducible, access-controlled, and bounded in time and cost, as regulated sectors require. Yet AOPs are typically kept as coarse-grained files, in which steps bundle actions, copied procedures drift apart, any step may reach any tool, and refactoring goes unchecked. We present \textbf{Functional Agent Operating Procedures (FAOP)}, a formal theory that brings database-grade discipline to AI agent workflows. FAOP recasts each workflow as a \emph{typed function over scoped state} and subjects an entire library to four progressive \textbf{normal forms}: Decomposed (FNF-1) splits a workflow into atomic steps, Typed (FNF-2) makes its data flow explicit and type-checked, Factored (FNF-3) gives each shared sub-workflow a single canonical definition so a change propagates instead of drifting, and Scoped (FNF-4) checks that each FAOP lists resources used by its own actions and the FAOPs it calls, while guardrails declared by it and by FAOPs that call it remain attached to their sources. Once normalized, every atomic step is a separate entity with a declared contract and can be type-checked, scope-confined, bounded, tested, traced, and audited on its own. A \textbf{FAOP algebra} of provably sound equalities turns refactoring and parallelization into compiler rewrites; orthogonal extensions add optimization (+O), bounded completion with time and cost limits (+B), and reachability checks (+R); and an architectural-level classifier (0--3) determines whether a workflow needs AI at all. Together, these make integrity, isolation, typing, stopping limits, cost, and optimization compile-time concerns across the whole library, not runtime hopes. We implement the theory in \textbf{SkillBase}, an AOP management system (what PostgreSQL is to data, SkillBase is to AOPs), and a separate executor that loads FAOPs from SkillBase and enforces these guarantees at run time. A worked financial workflow illustrates the approach end to end.
\end{abstract}

\vspace{1.5em}

\end{@twocolumnfalse}
]
\makeatother

\section{Introduction}

AI agents are entering the operational core of organizations, most visibly in regulated industries. A bank runs an agent to onboard a client through KYC, an insurer to investigate a claim for fraud, a hospital to secure prior authorization for a procedure, a law firm to review a contract for non-standard clauses, and so on. In each case a decision is acceptable only when the institution can show exactly how it was reached (which steps ran, in what order, on which data, under which permissions, and with what result), can reproduce it for a regulator, and can bound its time and cost. These are conditions of deployment, not optional qualities, yet none is supplied by the language model itself; each is a property of the multi-step \emph{workflow} built around it.

The reusable unit of agent behavior is the \emph{skill}: an artifact stored outside the model's weights that tells an agent how to perform a task, and that can be versioned, retrieved, and executed without retraining \citep{badhe2026survey,lu2025octotools}. Skills range from a single wrapped tool call to a \emph{composite} procedure with its own control flow \citep{badhe2026survey}. This paper concerns the composite, operational kind: the multi-step procedure an institution would otherwise write as a Standard Operating Procedure (SOP) for a person to follow. We call its agent-executable form an \emph{Agent Operating Procedure} (AOP).\footnote{\emph{AOP} denotes Agent Operating Procedure, unrelated to aspect-oriented programming.} An AOP is thus a workflow-shaped skill, and its steps are themselves skills: primitive actions, or calls to other AOPs.

Every AOP has the same shape:

\[
\begin{aligned}
\text{AOP}=\langle&\text{name},\text{description},\text{inputs},\\
&\text{outputs},\text{body}\rangle .
\end{aligned}
\]

The name and description make it discoverable; the inputs and outputs type its interface; the body holds the ordered steps.

Figure~\ref{fig:executor-pattern} shows how an AOP is run. A single generic agent executor (commonly a loop of reason, act, and observe \citep{yao2023react}, though the exact implementation varies by team) takes two inputs that change per workflow: an AOP describing the steps, and a list of tools the agent may call. A deployed system usually adds a user or agent interface and a gateway that mediates access to external agents and tools.

\begin{figure}[t]
  \centering
  \includegraphics[width=0.95\linewidth]{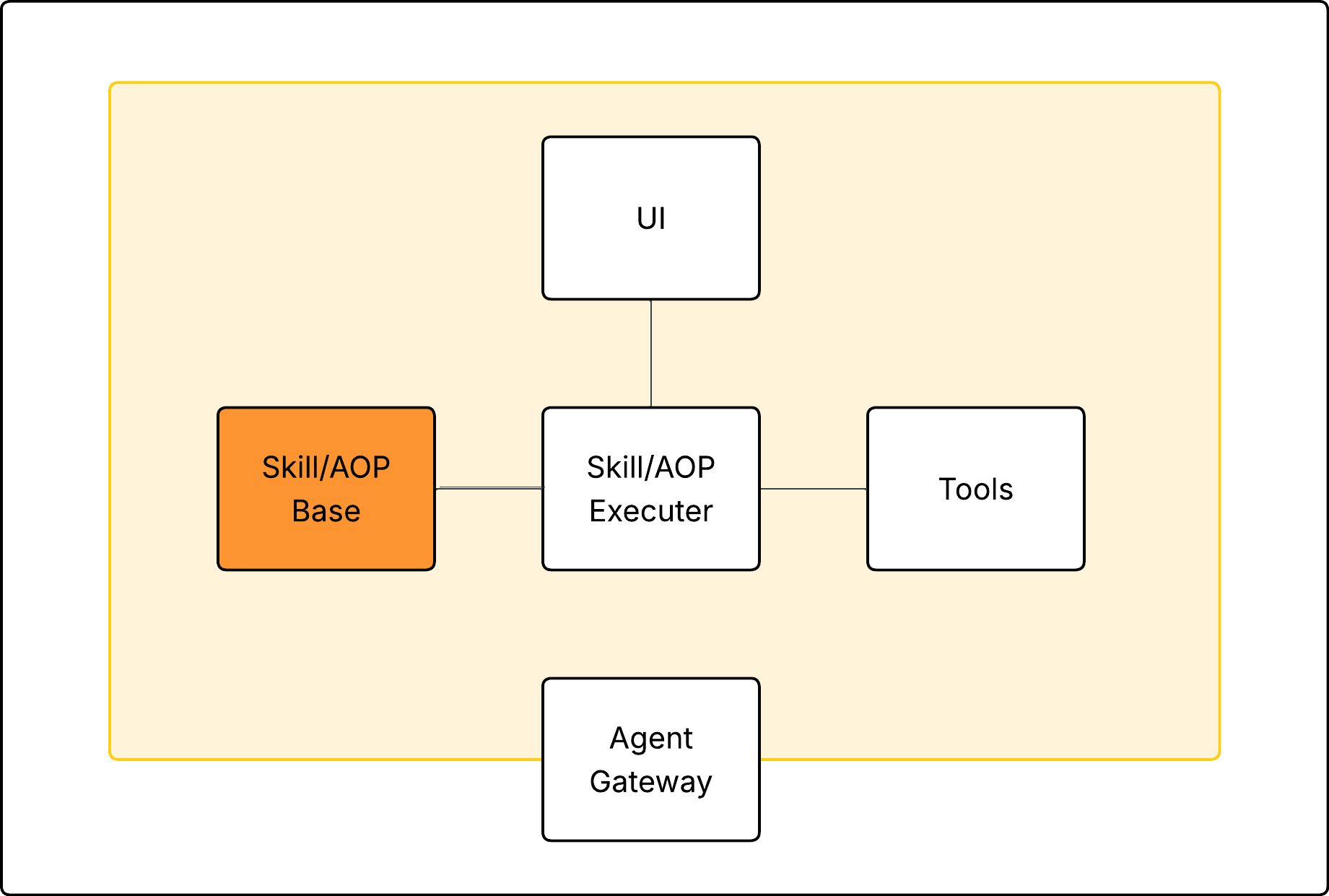}
  \caption{The AOP executor pattern. A single generic agent executes any workflow; only the AOP body and the tool list change per use case.}
  \label{fig:executor-pattern}
\end{figure}

Because an AOP lives outside the model's weights, proprietary procedures stay confidential rather than baked into external weights, and a procedure survives model upgrades and provider changes. Because the executor never changes, \emph{what} a workflow does is separated from \emph{how} it runs: the same executor runs onboarding, compliance checks, payment processing, and more; shipping a new agent collapses to authoring one AOP and its tool list, with no bespoke orchestration code to write, test, or deploy; and the fixed executor is the single point at which typing and tool access can be enforced. And because an AOP can call other AOPs, a shared sub-workflow can be written once and reused across workflows rather than reimplemented. These properties are why organizations build \emph{libraries} of reusable AOPs that agents draw on at runtime, increasingly populated not only by hand but by generators that emit many from natural-language descriptions or existing operational knowledge (Section~\ref{sec:prompt-engineering}). How an AOP is selected from such a library at run time is a separate problem, studied as skill routing \citep{zheng2026skillrouter}; we do not address it here. Our subject is the body: its structure, its composition, and the guarantees a compiler can establish over it before the workflow runs.

This design works in the small but strains at scale. A bank may run thousands of AOPs, yet in practice their shared sub-workflows (client lookup, compliance checks, document validation) are copy-pasted as text, so a single new compliance rule must be re-applied by hand in many places.

\subsection{The Problem: AOPs at Scale}
\label{sec:problem}
AOPs are commonly authored, tested, deployed, observed, and governed as opaque monoliths, so nothing below the whole workflow can be reused, type-checked, tested, scoped, or traced on its own. As a library grows, five kinds of defect follow:
\begin{itemize}
\item \textbf{Integrity across the library.} Copies of shared sub-workflows drift out of sync; inter-AOP dependencies are invisible, so a change breaks dependents without warning; and with no shared vocabulary (\texttt{client\_id} versus \texttt{id\_client}) nothing composes.
\item \textbf{Correctness within a workflow.} State is often untyped, so a step can read a field no earlier step produced; loops may run unbounded; branches may be unreachable and steps may compute results that nothing uses; and no time or cost bound is declared or checked.
\item \textbf{Isolation.} Tools and data are commonly granted to the AOP as a whole rather than to the steps that need them, and the only ``boundary'' is a comment the model may or may not honor, so a single prompt injection in one step can reach all of them.
\item \textbf{Composition.} Reordering, parallelizing, or extracting steps into a shared AOP is done by hand, with no proof that the workflow still does the same thing.
\item \textbf{Governance.} Whether a workflow needs AI at all (no model, a model as a tool, or a planner) is decided by intuition, with no formal basis.
\end{itemize}
Two of these defects are already measured in public skill marketplaces. For integrity, a study of 40{,}285 published skills finds widespread duplication: the same user intent is often published as several separate skills \citep{ling2026claudeskills}. For isolation, a security analysis of 31{,}132 skills finds that 26.1\% contain at least one vulnerability, most often data exfiltration (13.3\%) and privilege escalation (11.8\%) \citep{liu2026skillswild}.

\subsection{A Relational Analogy}

\begin{table*}[t]
\centering
\caption{A flat, pre-normalization loan file.}
\label{tab:loan-flat}
\small
\begin{tabularx}{\linewidth}{@{}l l X l l l r@{}}
\toprule
LoanID & Customer & Address & Phone & Branch & BranchMgr & Amount \\
\midrule
L001 & Alice & 10 Main St & 555-1234 & Downtown & Diana & 50,000 \\
L002 & Alice & 10 Main St & 555-1234 & Downtown & Diana & 25,000 \\
L003 & Carol & 7 Oak Ave & 555-5678 & Uptown & Dave & 80,000 \\
\bottomrule
\end{tabularx}
\end{table*}

AOP libraries are not the first to fail this way. Before Codd \citep{codd1970}, data lived in flat files in which every row repeated everything (Table~\ref{tab:loan-flat}). Such files had four structural defects: \emph{redundancy} (Alice's address and branch manager repeat on every row), \emph{broken integrity} (update one row, miss another, and the file contradicts itself), \emph{uncontrolled access} (anyone who reads the file reads everything), and \emph{no optimization} (every query scans every row). The answer, begun by Codd, was not a better file but a \emph{theory}: normal forms, a relational algebra, and schemas that enforce integrity and access by construction.

An AOP library has the same four defects: copied procedures (redundancy), copies that drift apart (broken integrity), steps that may reach any tool (uncontrolled access), and no checked way to refactor or parallelize (no optimization). \textbf{FAOP}, the theory this paper introduces, gives it the same three answers (normal forms, an algebra, and typed, scoped contracts in place of schemas) and the same new role: the \emph{Skill Architect}, who designs, normalizes, and governs an organization's AOP library, the DBA of the agentic era. In this analogy, FAOP is relational theory and SkillBase, our implementation (Section~\ref{sec:toolchain}), is PostgreSQL. The theory stands alone and can be applied by hand; the tooling makes it practical at scale.

\subsection{Contributions}

Our contributions are as follows:

\begin{enumerate}
\item A formal definition of FAOPs as typed, composable functions over scoped state, with explicit scope boundaries per step (Section~\ref{sec:formal-framework}).
\item Four progressive normal forms (FNF-1 through FNF-4) that incrementally guarantee atomicity, type safety, composition without duplication over an acyclic call graph that removes the recursion anomaly (a workflow that silently calls itself and never terminates), and scope validation: FNF-4 checks that each FAOP lists the resources required by its own actions and the FAOPs it calls, and preserves guardrail declarations from FAOPs that call it (Section~\ref{sec:normal-forms}). Additional behavioral properties, namely semantics-preserving optimization (+O), bounded completion with time and cost limits (+B), and reachability checks (+R): every branch can run under the declared contracts, and every computed result is used, are formalized as orthogonal extensions on top of the four normal forms.
\item A FAOP algebra with equational laws (commutativity of independent steps, associativity of sequential composition, branch distribution, and fold/unfold) that enable meaning-preserving refactoring and automatic parallelization of independent steps (Section~\ref{sec:algebra}).
\item An architectural-level classification (Level 0--3) computed from the normalized body, distinguishing a deterministic workflow, an LLM-as-tool agent, and selector/generator planners, giving organizations a formal basis for deciding whether a use case needs AI and what governance it calls for (Section~\ref{sec:architectural-levels}).
\item SkillBase, a reference implementation of an \emph{AOP management system} (a vocabulary, a compiler, and a registry). The compiler checks every normal form and extension before any runtime artifact is generated, and compiles a normalized AOP into an executable blueprint for the FAOP executor; the registry's impact analysis identifies every affected workflow when a shared AOP changes, so each is notified and reviewed again before adopting the change (Section~\ref{sec:toolchain}).
\end{enumerate}

\section{The FAOP Formalism}
\label{sec:formal-framework}

This section gives the formal definition of FAOPs, the primitive operators they are built from, and how those primitives combine into composite FAOPs.

\subsection{Definition}

Let $\mathcal{V}$ be a fixed vocabulary of typed field names. A \emph{state} $\sigma$ is a finite map from a subset of $\mathcal{V}$ to typed values, and a \emph{schema} $S\subseteq\mathcal{V}$ is a set of fields, such as the fields a step reads or writes.

\begin{definition}[FAOP]\label{def:faop}
A \emph{Functional Agent Operating Procedure} is an 8-tuple
\[
F=(\text{id},\ n,\ d,\ \Sin,\ \Sout,\ \Sopt,\ b,\ \Sigma),
\]
where $\text{id}$ is a unique versioned key in the \emph{registry} $\mathcal{R}$, the normalized library of FAOPs; $n$ and $d$ are a human-readable name and a semantic description used for discovery; $\Sin$, $\Sout$, and $\Sopt$, also written $\Sin(F)$, $\Sout(F)$, and $\Sopt(F)$, are its declared \emph{interface}: the fields it reads from its caller, the fields it always returns, and its optional outputs, which it returns only on some paths (Convention~\ref{conv:accumulation}); $b$, also written $b(F)$, is the \emph{body}, a term in the grammar of Section~\ref{sec:composite}, which FNF-2 checks against the interface (Section~\ref{sec:fnf2}); and $\Sigma$ is the \emph{scope}, comprising the resources named by the body and the guardrails declared for $F$ (Section~\ref{sec:fnf4}).
\end{definition}

\begin{convention}[State threading and scope]
\label{conv:accumulation}
The workflow carries a \emph{state}: a record mapping field names to values, threaded through the steps as in a functional program: each step reads the fields it needs and adds the fields it produces, then passes the updated record on. Every FAOP $F$ runs in its own state scope, like a function call: it sees only the fields its caller passes in ($\Sin(F)$) together with the local fields it computes itself, never another FAOP's internals, and it returns to the caller only its declared outputs: the fields of $\Sout(F)$, which it always produces, and any \emph{optional outputs} $\Sopt(F)$, fields it writes on only some paths, which return only when written. If two steps in the same body write the same field, the later one wins; we write $\sigma\uplus\sigma'$ for this update.
\end{convention}

\subsection{Primitive Operators}
\label{sec:primitives}

FAOP fixes six primitives (Table~\ref{tab:primitives}), together with a $\callop$ form that references a registered FAOP by identity. Let $f,g$ range over bodies and $p:S\to\{\texttt{true},\texttt{false}\}$ over typed predicates.

\begin{table}[!ht]
\centering
\caption{Primitive FAOP operators.}
\label{tab:primitives}
\small
\begin{tabularx}{\linewidth}{@{}l l X@{}}
\toprule
Primitive & Notation & Reading \\
\midrule
Action & $a:S\to S'$ & any single unit of work (for example, a tool call) \\
Sequence & $\seqop(f,g)$ & apply $f$, then $g$ \\
Branch & $\ifop(p,f,g)$ & apply $f$ if $p(\sigma)$, else $g$ \\
Loop & $\loopop(f,p)$ & apply $f$ until $p(\sigma)$ \\
Interrupt & $\iota(q,k)$ & suspend with query $q$; resume by passing the response to continuation $k$ \\
Checkpoint & $\chi$ & persist the running state to internal storage \\
Call & $\callop(\text{id})$ & call the registered FAOP $\text{id}$ in its own state scope (a function call) \\
\bottomrule
\end{tabularx}
\end{table}

\subsection{Composite FAOPs}
\label{sec:composite}

A \emph{body} is a term in the grammar
\[
\begin{aligned}
b ::= {}& a \mid \seqop(b_1,\ldots,b_n) \mid \ifop(p,b_1,b_2)\\
&{}\mid \loopop(b,p) \mid \iota(q,k) \mid \chi \mid \callop(\text{id}).
\end{aligned}
\]

A FAOP is a \emph{leaf} when its body has no sub-bodies (an action $a$, interrupt $\iota$, checkpoint $\chi$, or call $\callop(\text{id})$), and a \emph{composite} when its body is a $\seqop$, $\ifop$, or $\loopop$ combining sub-bodies. Each $\callop(\text{id})$ resolves, in the ambient registry $\mathcal{R}$, to the body of $\mathcal{R}[\text{id}]$.

\paragraph{Structural graphs.} The grammar and registry give rise to four graphs the paper keeps strictly separate. Three are structural: the \emph{body tree} is one FAOP's abstract syntax tree (internal nodes $\seqop$, $\ifop$, $\loopop$; leaves $a$, $\iota$, $\chi$, or $\callop(\text{id})$), finite and acyclic by construction; the \emph{execution graph} is a single runtime trace, in which cycles arise only from $\loopop$ and termination is governed by the loop predicate; and the \emph{call graph} below is the one later sections depend on. A fourth, the \emph{data-dependence graph}, puts an edge between two steps that conflict, whether through a record field or a shared external resource. It records which steps depend on which, so steps with no path between them are independent and may be reordered or run in parallel.

\begin{convention}[Call graph]
\label{conv:call-graph}
The call graph has one node per registered FAOP and an edge $F\to F'$ whenever $b(F)$ contains $\callop(\text{id}(F'))$, however deeply nested; a FAOP shared by several workflows has several incoming edges. The call graph must be acyclic: no FAOP may call itself, directly or through other FAOPs.
\end{convention}
Acyclicity costs no expressiveness: a recursive procedure can be rewritten as a $\loopop$ that keeps in the state what the recursion kept on its call stack, so its cycles live in the execution graph rather than the call graph. It also lets the compiler work callees-first when it checks, compiles, and answers impact queries, and it rules out recursive call cycles. It does not by itself bound execution: limits on loops, external calls, and waits come from +B.

\paragraph{Referential transparency.} Because $\mathcal{R}$ resolves each identifier to exactly one body, a call can be replaced by that body (inlining), and a registered body by a call to it (extraction), anywhere without changing $\Sin$, $\Sout$, or behavior; the inlined copy still runs in its own state scope (Convention~\ref{conv:accumulation}). Because the call graph is acyclic, inlining every call ends in a finite tree of primitives. We rely on this to justify the algebraic rewrites of Section~\ref{sec:algebra}.

\paragraph{Expressiveness.} The grammar is expressive enough to encode any workflow: by the structured program theorem of B\"{o}hm and Jacopini~\citep{bohm1966}, sequence, selection, and iteration over atomic actions (which FAOP provides as $\seqop$, $\ifop$, $\loopop$, and the action $a$) express any control flow drawable as a flowchart. The remaining FAOP primitives, $\iota$ and $\chi$, and the call form extend this core: $\iota$ adds external suspension and resumption for human or agent input, $\chi$ adds state persistence, and each $\callop(\text{id})$ abbreviates a registered body. None of these narrows what the core can express, so the B\"{o}hm--Jacopini completeness result still holds for FAOP.

\section{Normal Forms and Extensions}
\label{sec:normal-forms}

Normalization transforms an AOP library into an equivalent one, without changing the work it denotes, to remove a specific class of anomaly. Its input is the firm's whole library of source artifacts (Standard Operating Procedures, Agent Operating Procedures, or framework skills), typically many interconnected workflows that share steps, call one another, and evolve independently. Each is usually written as prose or thin pseudo-code for an LLM to follow, and the output is a registry of typed, composable FAOPs that a compiler can check. The four normal forms apply in order, each taking the previous as given. Like relational normal forms, which remove update, insertion, and deletion anomalies, each removes anomalies of its own:
\begin{itemize}
\item FNF-1 removes the \emph{compound-step anomaly} (a leaf that bundles several effects, so no single effect can be tested, traced, retried, or priced on its own) by splitting it into atomic steps.
\item FNF-2 removes the \emph{missing-field anomaly} (a step that reads a field no earlier step produced) by typing the data flow.
\item FNF-3 removes the \emph{duplication anomaly} (a shared sub-workflow copied across workflows and drifting out of sync, the workflow form of the update anomaly) by factoring it into one canonical FAOP the others call; it also rejects the \emph{recursion anomaly} (Section~\ref{sec:fnf3}).
\item FNF-4 checks that each FAOP lists the resources required by its own actions and the FAOPs it calls, while guardrails from FAOPs that call it remain attached to their declaring FAOP. This prevents incomplete resource declarations and keeps caller guardrails distinct when a shared FAOP is called from different places.
\end{itemize}

\textbf{Running example.} Examples in this section are drawn from one synthetic finance workflow, \texttt{issue\_statement.v1}, which issues a monthly account statement; Appendix~\ref{app:normalization-examples} works it end to end, one normal form at a time.

\subsection{FNF-1 Decomposed}
\label{sec:fnf1}
{
FNF-1 makes every step of a workflow small enough to stand on its own: each leaf may perform at most one operation that reaches outside it, such as one API call or one model call. In a prose AOP, a single step may hide several operations, and which ones it performs is known only by running it: the model reads the step and decides. FNF-1 makes every operation a separate, declared leaf. The compiler then knows before any run which leaves call a model and which do not, the executor can run every non-model leaf without the model, and each leaf can be tested, retried, bounded, and scoped on its own. FNF-1 takes a library of source artifacts (prose SOPs, AOPs, or framework skills) and produces a registry of FAOPs in which every leaf is atomic.

Atomicity is defined in terms of effects. An \emph{effect} is one operation a leaf performs on something outside itself, and every leaf is labeled with one symbol from a finite \emph{effect signature}; the baseline used in this paper is
\[
\begin{aligned}
\Sigma_E=\{&\opHTTP,\opDB,\opLLM,\opTool,\\
&\opFile,\opQueue,\opRet,\opPlan,\opCompute,\nopsym\}.
\end{aligned}
\]
The first six symbols name one HTTP request, one database read or write, one LLM call, one invocation of a registered tool, one file read or write, and one queue message. $\opRet$ names one read-only retrieval query against a knowledge base or search index, returning the retrieved items without calling a model; retrieval-augmented generation is then an $\opRet$ leaf followed by an $\opLLM$ generation step. A \emph{resource} is anything an action leaf names, such as a tool, endpoint, store, model, file, or queue; a value leaf names no external resource. $\opPlan$ marks a planning leaf, which produces a FAOP by selecting existing FAOPs or generating a new body (Section~\ref{sec:architectural-levels}), and $\nopsym$ is the empty step. $\opCompute$ marks a \emph{value leaf}: a pure function from the leaf's declared inputs to its declared outputs. It may do any amount of data work, such as translating one large JSON document into another, filtering or enriching records, or computing totals and flags, but it reads nothing beyond its inputs, writes nothing beyond its outputs, and calls no API, database, or model. A domain may add symbols for its own units of work: a managed retrieval-augmented generation endpoint, for instance, can be labeled $\opRag$ so that a single call that retrieves and generates behind one API is tracked as its own effect.

\begin{definition}[Atomic leaf]\label{def:atomic-leaf}
A leaf $\ell$, a FAOP in its own right with its own $\text{id}$, $n$, $d$, $\Sin$, and $\Sout$, is \emph{atomic} iff it performs at most one effect, $|\operatorname{effects}(\ell)|\leq 1$. An \emph{action leaf} performs one; a \emph{value leaf} performs none. In the normalized registry, every leaf other than a checkpoint is registered under its own id, and bodies refer to it by $\callop(\text{id})$; checkpoints are written inline.
\end{definition}

\begin{definition}[FNF-1]
Writing $\leavesop(F)$ for the leaves of the body of $F$,
\[
\operatorname{FNF\text{-}1}(\mathcal{R})\iff\forall F\in\mathcal{R}:\ \forall \ell\in\leavesop(F):\ \ell\text{ is atomic}.
\]
\end{definition}

\begin{lemma}[Lossless leaf split]\label{lem:leaf-split}
Let a leaf $\ell$ perform effects $e_1,\ldots,e_k$ in this order ($k>1$). Replacing $\ell$ by $\seqop(\ell_1,\ldots,\ell_k)$, where $\ell_i$ performs $e_i$ alone, does not change what the workflow does: the same effects run in the same order, on the same inputs, with the same outputs.
\end{lemma}

The lemma holds because a leaf already runs its effects one after another, unconditionally, and $\seqop$ does the same with each effect in its own leaf; any value $\ell$ passed internally from one effect to the next becomes a field passed from $\ell_i$ to $\ell_{i+1}$. The split replaces one node of the body tree with a sequence in the same position and changes nothing around it. The surrounding constructs decide when that position runs and how often: only when its $\ifop$ arm is taken, once per iteration of an enclosing $\loopop$, and for every caller if it sits in a shared FAOP. The new sequence therefore runs exactly when, and as often as, the old leaf did.

A workflow reaches FNF-1 in two steps, performed by the Skill Architect (a human or an AI agent) and guided by the compiler. The architect first makes the steps explicit by rewriting the source in the grammar of Section~\ref{sec:composite} (sequential text becomes $\seqop$, alternatives $\ifop$, repetition $\loopop$, a human question $\iota$, a persistence point $\chi$, a known procedure $\callop(\text{id})$) and labels each leaf with its effect. This gives an initial registry $\mathcal{R}_0$ whose leaves may still bundle several effects; output from workflow generators such as AutoFlow \citep{li2024autoflow}, AFlow \citep{zhang2025aflow}, and Flow-of-Action \citep{pei2025flowofaction} enters at this point. The compiler then checks every FAOP in one $O(\|F\|)$ pass, where $\|F\|$ is the number of steps in the body of $F$, and reports each leaf that is not atomic:
\begin{lstlisting}
FNF1Check(F):
- for each leaf L in the body of F:
-   report L unless it declares exactly
  one effect symbol from Sigma_E
- accept iff no leaf is reported
\end{lstlisting}
The architect splits each reported leaf by the lemma, and the compiler checks again. The process terminates, since each split reduces the number of multi-effect leaves. By the lemma, every workflow in the result performs the same effects, in the same order, on the same inputs, with the same outputs as in $\mathcal{R}_0$; the only difference is that every leaf is now atomic.

The compiler checks the declared labels, not the code behind them; whether a leaf labeled $\opHTTP$ really issues one request is the architect's responsibility. In the running example, a step that reads both transactions and interest splits into two leaves (Appendix~\ref{app:fnf1}). All later normal forms reason over these atomic leaves, so most benefits in Table~\ref{tab:benefits} rest on FNF-1.
}

\subsection{FNF-2 Typed}
\label{sec:fnf2}
{
Every FAOP declares its interface before its body runs: $\Sin(F)$, the fields it requires, $\Sout(F)$, the fields it promises, and $\Sopt(F)$, the fields it may also return (Definition~\ref{def:faop}). FNF-2 is the condition that the body honors that interface on \emph{every} run: no step reads a field that has not yet been produced, and the body produces every promised output. For example, FNF-2 rejects a workflow that computes a balance before any step has fetched the transactions (Appendix~\ref{app:fnf2}).

The condition is decided by a standard forward dataflow analysis over the body~\citep{kildall1973,cousot1977}. It propagates the \emph{available schema} $S$: the typed fields guaranteed to be in the state on every run that reaches a given point. Writing $\mathcal{A}(b,S)$ for the schema available after body $b$ entered with $S$:
\[
\begin{aligned}
\mathcal{A}(\ell,S)&=S\cup\Sout(\ell) &&\text{(leaf)}\\
\mathcal{A}(\callop(\text{id}),S)&=S\cup\Sout(\mathcal{R}[\text{id}]) &&\text{(call)}\\
\mathcal{A}(\seqop(b_1,\ldots,b_n),S)&=\mathcal{A}(b_n,\ldots\mathcal{A}(b_1,S)) &&\text{(seq)}\\
\mathcal{A}(\ifop(p,f,g),S)&=\mathcal{A}(f,S)\cap\mathcal{A}(g,S) &&\text{(if)}\\
\mathcal{A}(\loopop(f,p),S)&=\mathcal{A}(f,S). &&\text{(loop)}
\end{aligned}
\]
A step or predicate $x$ reached with available schema $S$ is well-typed iff $\Sin(x)\subseteq S$. The $(\text{call})$ rule reads the callee's declared $\Sout$ and never its body, so each FAOP is checked once against its interface, as in separate compilation. It ignores optional outputs, since a call returns them only on some paths.

\begin{definition}[FNF-2]\label{def:fnf2}
$\operatorname{FNF\text{-}2}(\mathcal{R})$ iff for every $F\in\mathcal{R}$, applying these rules to $b(F)$ from $\Sin(F)$,
\begin{enumerate}
\item every step and predicate it reaches is well-typed there, and
\item $\Sout(F)\subseteq\mathcal{A}(b(F),\Sin(F))$.
\end{enumerate}
\end{definition}

The two conditions of Definition~\ref{def:fnf2} are the workflow forms of two rules a statically checked language such as Java enforces~\citep{gosling2021jls}. Condition~1, $\Sin(x)\subseteq S$, is \emph{definite assignment}: a variable may be read only where it is assigned on every path reaching the read. Condition~2 is the rule that a method declaring a result may not complete without returning one. Definite assignment bites hardest at a branch, since a field produced by one arm alone is not in the schema after the join: it is available to the rest of that arm but not past it. The remedies are the usual ones, namely move the reader into that arm, have the other arm produce the field too, or produce a default before the branch. A checker could instead keep such a field available past the join by remembering the predicate under which it was produced, admitting a read only under a later test of that same predicate; this is occurrence typing~\citep{tobinhochstadt2010}. FNF-2 does not do so, and Proposition~\ref{prop:fnf2} holds without it.

Two facts justify the $(\text{loop})$ equation, $\mathcal{A}(\loopop(f,p),S)=\mathcal{A}(f,S)$. Checking the body once, under the schema $S$ at loop entry, covers every iteration, so no fixpoint iteration is needed: fields are only ever added, so a later iteration starts with at least $S$. And $\mathcal{A}(f,S)$ is still available after the loop because a FAOP loop is a do-while (Section~\ref{sec:primitives}), so the body runs at least once. A loop that may run zero times is written $\ifop(q,\loopop(f,p),\nopsym)$, where the $(\text{if})$ rule correctly withholds those fields.

The compiler decides FNF-2 in one syntax-directed pass per body, in $O(\|F\|\cdot|\mathcal{V}|)$:
\begin{lstlisting}
FNF2Check(F):
- S = S_in(F)
- walk b(F) in order, updating S by the
  rules above
- report every field a step or predicate
  reads that is not in S
- report every field of S_out(F) that is
  not in the final S
- accept iff nothing is reported
\end{lstlisting}
Each report names the missing field and the schema available at that point. The Skill Architect (a human or an AI agent) resolves it by reordering steps, adding a step that produces the field, or correcting the field's name.

\begin{proposition}[Type safety]\label{prop:fnf2}
If \textup{FNF2Check} accepts $F$, then on every run of $F$ each step and predicate finds every field it reads already in the state, and $F$ returns every field of $\Sout(F)$.
\end{proposition}
\begin{proof}
By induction on $b(F)$: each rule adds to $S$ only fields produced on every path through that construct, since a sequence runs all its steps, a branch contributes only what both arms produce, and a loop runs its body at least once. Fields are only ever added to the state, so a field in the available schema at a point is in the state when control reaches it. Condition~1 then makes every read succeed and condition~2 makes the final state contain $\Sout(F)$.
\end{proof}

This and the later propositions assume that each leaf produces what it declares. The executor enforces this at run time: it checks the inputs of $F$ on entry and the outputs of each leaf against its $\Sout$.
}

\subsection{FNF-3 Factored}
\label{sec:fnf3}
{
One step can recur across a whole library: reading an account and confirming it is active is needed when issuing a statement, closing an account, and handling a dispute. If each workflow carries its own copy, a change must be repeated in each one, the copies drift apart, and nothing records which workflows a change affects; libraries that accept every proposed skill pile up such entries~\citep{li2026dynamicskills}. FNF-3 requires that such a sub-workflow exist once, as a registered FAOP the others call; it also finds FAOPs that do the same job with different steps or wording, so that the architect can merge them.

FNF-3 looks for two kinds of copies, both shown on one example below: exact copies, the same steps written in more than one place, and near copies, registered FAOPs that do the same job with different steps or wording.

\paragraph{Exact copies.} Each letter below is one step, registered once under its own name, and \texttt{W1} to \texttt{W4} are composite FAOPs:
\begin{lstlisting}[basicstyle=\ttfamily\small,mathescape]
W1 = seq(T, I, B)
W2 = seq(A, call(W1), M)
W3 = seq(A, T, I, B, P)
W4 = seq(A$'$, T$'$, D, I, B)
\end{lstlisting}
A \emph{sub-workflow} $s$ is one step of a body, or several consecutive steps of one sequence. A sequence written directly inside another is treated as part of it, so \texttt{seq(A, seq(T, I), B)} has the same consecutive steps as \texttt{seq(A, T, I, B)}; a call, however, stays one step. Two steps are identical when they do the same thing: the same kind, effect, tool, action (its query or prompt), input and output fields, condition, loop settings, and called FAOP, and identical inner steps. A step's own name and description are labels for people and are not compared; its fields are compared by their names in the shared vocabulary. Two sub-workflows are identical, $s\equiv s'$, when their steps are identical one by one. Identical sub-workflows at different places, other than single calls or checkpoints, are \emph{exact copies}, and a sub-workflow with an exact copy is a \emph{repeat}.

In the example, \texttt{T I B} in \texttt{W3} is an exact copy of the body of \texttt{W1}. Suppose, for example, that \texttt{T}$'$ has a different label from \texttt{T} but does the same thing; then their bodies are exact copies. Conversely, a step with the same label as \texttt{B} but a different output field is not a copy of \texttt{B}. Labels cannot decide copies; what a step does can, and FNF-2 makes its fields reliable, since they are checked and drawn from the vocabulary. Exact copies are exact relative to this vocabulary: if it gives one concept two field names, the exact search misses that copy, though the semantic search, which reads the whole definition, can still find it as a near copy. \texttt{A} occurs twice, but a single call already refers to one shared FAOP.

\paragraph{Near copies.} Two FAOPs can also do the same job with different steps or wording. The compiler embeds the full definition of each registered FAOP, that is, its name, description, inputs, outputs, and body, and compares two FAOPs $F$ and $F'$ by a similarity $\operatorname{sim}(F,F')$ of their embeddings; any measure will do, and we use cosine similarity. $F$ and $F'$ are \emph{near copies} when $\operatorname{sim}(F,F')\ge\theta$, for a threshold $\theta$ that the Skill Architect sets. Suppose, for example, that \texttt{A}$'$ does the job of \texttt{A} with the same input and output fields but through a different tool; then the two are not identical, and if their similarity reaches $\theta$, they are near copies.

Factoring gives a sub-workflow one definition, owner, version, test boundary, and list of callers, but also makes it a shared contract that no single workflow may change alone. So not every copy should be removed. Every copy is a sub-workflow: an exact copy is part of a body, and a near copy $F$ takes part through its whole body $b(F)$. For two copies $s$ and $s'$, $\operatorname{distinct}(s,s')$ records the architect's decision to keep both, because they do different jobs or should evolve separately. It is a record of decisions, not something the compiler computes: it starts empty, grows only when the architect keeps a pair, and is symmetric, $\operatorname{distinct}(s,s')=\operatorname{distinct}(s',s)$. A pair that is factored instead needs no record, since only one copy remains. The \emph{candidate pairs} of $\mathcal{R}$ are its pairs of exact copies and the bodies of its pairs of near copies:
\[
\begin{aligned}
\operatorname{Cand}(\mathcal{R})=\{\{s,s'\}:{}&s\text{ and }s'\text{ are exact copies,}\\
&\text{or bodies of near copies}\}.
\end{aligned}
\]
\begin{definition}[FNF-3]\label{def:fnf3}
A registry $\mathcal{R}$ satisfies \emph{FNF-3} iff every candidate pair is distinct:
\[
\forall\{s,s'\}\in\operatorname{Cand}(\mathcal{R}):\ \operatorname{distinct}(s,s').
\]
\end{definition}
Factoring removes a pair, so in FNF-3 each shared sub-workflow exists once, as the body of a FAOP that its users call. Appendix~\ref{app:fnf3} shows one example.

\paragraph{Finding candidates.} The compiler finds the candidate pairs with two searches, one for each kind of copy: an exact search over the steps of all bodies, and a semantic search over whole registered FAOPs. Both costs are given in terms of the \emph{size} $\|\mathcal{R}\|$ of $\mathcal{R}$, the number of steps in all its bodies, where a branch or a loop counts as one step and so do the steps inside it; the number of FAOPs, $|\mathcal{R}|$, is at most $\|\mathcal{R}\|$, since every FAOP has at least one step.

The \emph{exact search} works in three steps. First, each step gets a number from what it does, never from its label: a hash table maps its kind, effect, tool, action, fields, and, for a branch or a loop, the numbers of its inner steps to a number, so identical steps get the same number. Second, every sequence, including the arms of a branch and the body of a loop, becomes a string of these numbers, one per step, ended by a separator of its own; a call is a single number, since calls are never expanded. Third, a suffix tree is built over all these strings. Each inner node of the tree stands for a run of numbers that occurs more than once, the leaves below it give all its occurrences, and every repeat is the start of such a run, so the tree holds every repeat together with its occurrences~\citep{gusfield1997}; clone detectors use the same structure to find copied code~\citep{koschke2006}. The compiler then shows the \emph{maximal repeats}, those that lose an occurrence if extended by one step on either side, longest first, each with its occurrences. In the example, these are \texttt{T I B}, in \texttt{W1} and \texttt{W3}; \texttt{I B}, which also occurs in \texttt{W4}; and the single step that forms the bodies of \texttt{T} and \texttt{T}$'$. \texttt{I B} needs its own decision because it also occurs outside \texttt{T I B}: factoring both, the longer first, gives a FAOP \texttt{W5 = seq(I, B)}, and the body of \texttt{W1} becomes \texttt{seq(T, call(W5))}. A shorter part that occurs only inside a longer repeat needs its own decision only if copies of the longer repeat are declared distinct. One run of the exact search finds all exact copies in the registry at once, in $O(\|\mathcal{R}\|)$ expected time: numbering takes $O(\|\mathcal{R}\|)$ expected time, the strings have $O(\|\mathcal{R}\|)$ entries in total, and the suffix tree is built and read in time linear in their length; see Gusfield~\citep{gusfield1997} for the construction.

The \emph{semantic search} embeds the full definition of each registered FAOP once, compares the embeddings of every pair, and reports the pairs whose similarity reaches $\theta$. This takes $|\mathcal{R}|$ embeddings and $|\mathcal{R}|(|\mathcal{R}|-1)/2$ comparisons; in a large library, an approximate nearest-neighbor index finds most close pairs much faster~\citep{malkov2020hnsw}. FNF-2 helps here too, since FAOPs that read and write the same fields get a higher similarity. The search may also give some similarity to FAOPs that only share a sub-workflow, such as \texttt{W1} and \texttt{W3}, but less, and without showing the shared steps; that is the job of the exact search. We embed only registered FAOPs: embedding every sub-workflow inside the bodies as well would need $O(\|\mathcal{R}\|^2)$ embeddings and $O(\|\mathcal{R}\|^4)$ comparisons, or about $O(\|\mathcal{R}\|^2\log\|\mathcal{R}\|)$ with an index.

The results of the two searches are combined, so a copy found by both, such as two FAOPs with identical bodies, appears once.

\paragraph{Checking FNF-3.}
\begin{lstlisting}[basicstyle=\ttfamily\small]
FNF3Check(R):
- for each pair {s, s'} of exact copies,
  or of bodies of near copies:
    report it unless distinct(s, s')
- accept iff nothing is reported
\end{lstlisting}
The check runs both searches and reports every candidate pair that has no $\operatorname{distinct}$ record; it accepts the registry if it reports nothing, which is exactly the FNF-3 condition. For this, neither search may miss a pair: the exact search finds every exact copy by construction, and the semantic search compares every pair of FAOPs directly instead of using the approximate index, which can miss a pair.

\paragraph{Deciding and applying.} For each pair $\{s,s'\}$, the architect either \emph{factors} it or declares $\operatorname{distinct}(s,s')$, and the decision is recorded. Factoring keeps one definition, that of $s$, and turns $s'$ into a call to it, in two moves:
{\begin{itemize}
\item If $s$ is already the whole body of a registered FAOP $C$, it stays as it is. Otherwise $s$ sits inside a larger body; it is registered as a new FAOP $C$ with body $s$, and replaced there by a call to $C$.
\item If $s'$ sits inside a larger body, it is replaced there by a call to $C$. If $s'$ is instead the whole body of a registered FAOP $C'$, every call to $C'$ becomes a call to $C$, and $C'$ is removed.
\end{itemize}
}
In the example, $s$ is the body of \texttt{W1}, which stays, and $s'$ is the copy of \texttt{T I B} inside \texttt{W3}, which becomes \texttt{call(W1)}. Factoring requires equal interfaces,
\[
\Sin(s)=\Sin(s'),\ \Sout(s)=\Sout(s'),\ \Sopt(s)=\Sopt(s'),
\]
that is, the same fields read from the caller, written on every path, and written on only some paths. A FAOP declares them; for an inline copy, they are computed from its steps: its inputs are the fields its steps read before any of its own steps writes them, its outputs are the fields it writes on every path (the FNF-2 rules give these), and its optional outputs are the other fields it writes. Exact copies meet this by construction, and with equal interfaces every caller stays well-typed. Exact copies consist of identical steps, so factoring them changes nothing that a workflow does; the architect decides only whether they should share one definition. Near copies differ, so the architect reads both in full and factors them only if they do the same job.

\paragraph{Rounds and cost.} The compiler works in rounds. In each round, it runs both searches and hands the architect, a human or an AI agent, all undecided candidate pairs at once; the architect decides every pair; and the compiler applies the decisions. Applying them changes bodies, and the changed bodies can contain copies that did not exist before, so the next round searches again. The compiler stops at the first round whose search finds no undecided pair. In the example, round~1 replaces the copy of \texttt{T I B} in \texttt{W3} with \texttt{call(W1)}, registers \texttt{W5 = seq(I, B)} for the copies of \texttt{I B} in \texttt{W1} and \texttt{W4}, makes the callers of \texttt{T}$'$ call \texttt{T}, and, if the architect finds that \texttt{A}$'$ does the job of \texttt{A}, makes its callers call \texttt{A}. This creates a new repeat: \texttt{A call(W1)} now occurs in both \texttt{W2} and \texttt{W3}. Round~2 factors it into a new FAOP \texttt{W6}, unless the architect declares the pair distinct. Round~3 finds no new copy, and the compiler stops with:
\begin{lstlisting}[basicstyle=\ttfamily\small]
W1 = seq(T, call(W5))
W2 = seq(call(W6), M)
W3 = seq(call(W6), P)
W4 = seq(A, T, D, call(W5))
W5 = seq(I, B)
W6 = seq(A, call(W1))
\end{lstlisting}
Here \texttt{T}$'$ and \texttt{A}$'$ are gone, no two sub-workflows other than single calls are identical, and each workflow runs the same steps as before, except that \texttt{W4} now uses \texttt{A} instead of \texttt{A}$'$.

The rounds end. To see why, count two things over all bodies: the \emph{neighboring pairs}, that is, the places where one step directly follows another in a sequence, so that \texttt{seq(T, I, B)} has two, \texttt{T} before \texttt{I} and \texttt{I} before \texttt{B}; and the steps that are not calls, such as an inline branch or loop or the action in the body of a step's own FAOP. Both counts start at most $\|\mathcal{R}\|$ and never grow. Replacing an inline copy of several steps with one call removes the neighboring pairs inside that copy, and replacing a single inline branch or loop with a call turns a non-call step into a call; either way one count falls, so this happens at most $2\|\mathcal{R}\|$ times. Factoring a pair of registered FAOPs removes one of them; there are $|\mathcal{R}|$ FAOPs at the start, and each factoring adds at most one, so this too happens a limited number of times. Finally, a round that factors nothing leaves every reported pair declared distinct and the registry unchanged, so the compiler stops. Hence there are $O(\|\mathcal{R}\|)$ rounds.

Each factoring adds at most one step, so the registry stays of size $O(\|\mathcal{R}\|)$. Over all rounds, the exact searches take $O(\|\mathcal{R}\|^2)$ expected time, and each search can start from the results of the previous one rather than from scratch~\citep{gode2009incremental}; the semantic search compares at most $|\mathcal{R}|(|\mathcal{R}|-1)/2$ pairs per round, fewer with the index. One run of \textup{FNF3Check} finds all candidate pairs in $O(\|\mathcal{R}\|+|\mathcal{R}|^2)$ expected time, $O(\|\mathcal{R}\|)$ for the exact search and $|\mathcal{R}|(|\mathcal{R}|-1)/2$ comparisons for the semantic search, plus $|\mathcal{R}|$ embeddings and the time to list the pairs it reports. Faster searches that find more copies are future work (Section~\ref{sec:future-work}).

Under FNF-3, each shared sub-workflow has one definition, so a change propagates instead of drifting, and the registry names every workflow the change affects. Because every shared sub-workflow is now a call, a FAOP that calls itself through a chain of calls appears as a cycle in the call graph, which the compiler rejects (Convention~\ref{conv:call-graph}); this removes the recursion anomaly.
}

\subsection{FNF-4 Scoped}
\label{sec:fnf4}
{
FNF-4 is the normal form in which every FAOP in a registry has a valid scope. Validity depends on both the FAOP's own actions and its place in the composition: a caller must account for resources required by the FAOPs it calls, while a called FAOP is subject to guardrails declared by itself and by the FAOPs that call it. A scope therefore has two components, resources and guardrails, and each component can contain local entries and entries imposed by composition:
\[
\Sigma(F)=\langle\,\mathcal{T}(F),\ \mathcal{G}(F)\,\rangle,
\]
where $\mathcal{T}(F)$ contains resource entries and $\mathcal{G}(F)$ contains the guardrails that $F$ itself declares. Local resource entries come from action leaves in $F$; composition-imposed resource entries come from FAOPs that $F$ calls. A \emph{call to a FAOP} is a place in a workflow that invokes its registered definition; the same FAOP may be called from several places. The \emph{route} $\pi$ of a call to $F$ is the chain of calls from a root workflow down to it; the guardrail entries at that call, $\mathcal{G}^*(\pi)$, are the local entries of $\mathcal{G}(F)$ together with the composition-imposed entries declared by the FAOPs that call $F$ along $\pi$. Every entry retains the identifier of the FAOP that declared it.

\paragraph{FNF-4-T: resource requirements.} First collect the resources named by action leaves directly in $F$'s body, tagging each entry with $F$'s identifier. These are the resources required by $F$'s own actions. Because $F$ may call other FAOPs, also include the resource entries required by each FAOP that $F$ calls, preserving the identifier attached to every entry. The resulting set must equal $\mathcal{T}(F)$; the checker rejects missing or extra entries.

\paragraph{FNF-4-G: guardrail propagation.} At each call to $F$, with route $\pi$, the expected entries $\mathcal{G}^*(\pi)$ are the declarations made by $F$ together with those made by the FAOPs that call $F$ along $\pi$. Every entry retains its declaring FAOP's identifier. For example, if $F_2$ and $F_3$ both call $F_1$, the expected entries at the call from $F_2$ are the declarations from $F_2$ and $F_1$; at the call from $F_3$, they are the declarations from $F_3$ and $F_1$. The checker compares the entries recorded for each call with the expected entries for its route, rejecting any missing, extra, or incorrectly attributed entry.

\begin{definition}[FNF-4]
A registry $\mathcal{R}$ is in FNF-4 iff (i) each $\mathcal{T}(F)$ contains exactly the resource entries from actions in $F$ and in the FAOPs it calls, with declaring identifiers retained, and (ii) at each call to $F$, with route $\pi$, $\mathcal{G}^*(\pi)$ contains exactly the declarations from $F$ and the FAOPs that call it along $\pi$, with declaring identifiers retained.
\end{definition}

Authors may supply the entries before compilation. The checker computes the expected resource entries in postorder, so each called FAOP is checked before its caller, then checks guardrails separately for every root-to-call route:
\begin{lstlisting}
FNF4Check(R):
- for each FAOP F in R:
-   required = resources used by F's own actions,
    tagged id(F), plus the required resources
    of every FAOP F calls, ids preserved
-   report F unless T(F) equals required
- for each call to a FAOP F in R:
-   required = guardrails declared by F and by the
    FAOPs that call F on this route, each tagged
    with its declaring id
-   report the call unless its guardrail entries
    equal required
- accept iff nothing is reported
\end{lstlisting}

FNF-4 checks that scopes are declared consistently and records where each entry comes from. It does not grant access or enforce anything by itself: deployment policy approves access, a conforming executor confines each call to its declared resources, and a runtime check, such as a validator, enforces each guardrail.

}

\subsection{FAOP Algebra}
\label{sec:algebra}

In universal algebra, an algebra is a carrier set with operations defined over it, and an equational theory adds laws that equate terms denoting the same value~\citep{burris1981}. For FAOP the carrier is the set of well-typed (FNF-2) bodies, the operations are the primitives of Section~\ref{sec:primitives}, and the laws are five equations (L1--L5) a compiler applies as directed rewrites. We fix each in turn, then prove the laws sound against a denotational semantics.

\paragraph{Signature.} Every body $f$ denotes a state transformer $\sem{f}$, a map from states to states (Definition~\ref{def:faop}, Convention~\ref{conv:accumulation}). Sequential composition $\seqop$ is the composition operator and the empty step $\nopsym$ its unit; branch $\ifop$, loop $\loopop$, and call $\callop$ are the control constructors. Parallel $\parop$ is not a primitive but a \emph{derived} operator, equal to $\seqop$ on independent steps (L3).

\paragraph{Semantics.} A state $\sigma$ is a finite map from field names to values, and $\uplus$ merges two states, the right one winning on any shared field (Convention~\ref{conv:accumulation}). Each leaf $\ell$ computes an \emph{output record} $\delta_\ell(\sigma)$: the fields it produces (its declared $\Sout(\ell)$), computed from the fields it reads (its declared $\Sin(\ell)$). Running $\ell$ merges that record into the state; likewise, for any body $x$, $\delta_x(\sigma)$ is the record of fields $x$ writes when run on $\sigma$, so $\sem{x}(\sigma)=\sigma\uplus\delta_x(\sigma)$. The state transformer of each body is:
\begin{flalign*}
\sem{\ell}(\sigma) &= \sigma\uplus\delta_\ell(\sigma) & \\
\sem{\nopsym}(\sigma) &= \sigma & \\
\sem{\seqop(a,b)}(\sigma) &= \sem{b}(\sem{a}(\sigma)) & \\
\sem{\ifop(p,a,b)}(\sigma) &= \sem{a}(\sigma)\ \text{if } p(\sigma),\ \text{else } \sem{b}(\sigma) & \\
\sem{\callop(\text{id})}(\sigma) &= \sem{\operatorname{frame}_{F}(b(F))}(\sigma),\quad F=\mathcal{R}[\text{id}] & \\
\sem{\parop(a,b)}(\sigma) &= \sigma\uplus\delta_a(\sigma)\uplus\delta_b(\sigma)\quad (a\perp b). &
\end{flalign*}
Here $\operatorname{frame}_{F}(b)$ runs a body $b$ the way a call to $F$ runs it, in $F$'s own scope: $b$ sees only the caller's fields that $F$ takes as inputs, only $F$'s outputs are copied back to the caller, and any other field $b$ writes stays local. It is notation for the call boundary, not a new primitive:
\[
\sem{\operatorname{frame}_{F}(b)}(\sigma)
=\sigma\uplus\left(\sem{b}(\sigma|_{\Sin(F)})\right)|_{\Sout(F){}\cup\Sopt(F)}.
\]
An optional output is copied back only if the body holds it at the end. Two bodies are \emph{equal} when they have the same record transformer and observable external behavior.

\paragraph{Independence.}
\begin{definition}[Independence]
\label{def:independence}
For a step $a$, let
\[
\begin{aligned}
\operatorname{Read}(a)&=\Sin(a)\cup\{\text{resources }a\text{ reads}\},\\
\operatorname{Write}(a)&=\Sout(a)\cup\Sopt(a)\cup\{\text{resources }a\text{ writes}\}.
\end{aligned}
\]
For a call, the resources include those read or written by the FAOPs it calls. An access whose mode is unknown counts as a write. Resources are named from the vocabulary, like fields, and two names that may reach the same data count as one: for example, \texttt{core\_banking} and \texttt{ledger\_svc} count as one if both can reach the same account records. Steps $a,b$ are \emph{independent}, $a\perp b$, when no read-after-write, write-after-read, or write-after-write conflict holds:
\[
\begin{aligned}
a\perp b \iff{}& \operatorname{Write}(a)\cap\bigl(\operatorname{Read}(b)\cup\operatorname{Write}(b)\bigr)=\emptyset\\
&\wedge\ \ \operatorname{Write}(b)\cap\operatorname{Read}(a)=\emptyset .
\end{aligned}
\]
Two steps that only read the same resource are independent only if both see the same data, for example a fixed version, or \emph{snapshot}, of it: otherwise a write from outside the workflow between the two reads could make their order matter. Any other undeclared interaction prevents reordering. These conservative conditions extend the classical field test~\citep{bernstein1966} to external effects.
\end{definition}
The independence test compares two steps at a time. To decide whether a step can move past \emph{several} others, we gather all pairwise conflicts into one graph and read the answer off its reachability.
\begin{convention}[Data-dependence graph]
\label{conv:dep-graph}
For a sequence $\seqop(s_1,\ldots,s_n)$, put one node per step and add an edge $s_i\to s_j$ ($i<j$) whenever $s_i\not\perp s_j$, whether the conflict is through fields or external resources. The edge means ``$s_j$ must run after $s_i$.'' The graph uses typed schemas and declared resource accesses, not inferred program behavior. It is the data-dependence part of the classical program dependence graph~\citep{ferrante1987}, which also records control dependences.
\end{convention}
The compiler may reorder two steps when neither reaches the other; an antichain may run in parallel, and every topological order preserves the declared dependencies. Building the graphs of all sequences in a body $F$ takes $O(\|F\|^2\cdot(|\mathcal{V}|+|\mathcal{T}(F)|))$ time: each pair of sibling steps is tested once, and each footprint holds at most $|\mathcal{V}|$ fields and $|\mathcal{T}(F)|$ resources. An earliest-layer parallel schedule then takes time linear in the number of nodes and edges, which is $O(\|F\|^2)$. Soundness requires honest schemas and resource declarations, including aliases. Call frames hide local fields, but do not isolate external services or storage: those interactions must remain in the graph. Parallel execution also requires the declared read-only and snapshot contracts to hold at run time.

\paragraph{Laws.} For all bodies $a,b,c$ and predicate $p$, equality means the same record transformer and observable external behavior:
\begin{description}
\item[(L1) Identity.] $\seqop(\nopsym,f)=f=\seqop(f,\nopsym)$, with $\nopsym$ the empty do-nothing step.
\item[(L2) Associativity.] $\seqop(\seqop(a,b),c)=\seqop(a,\seqop(b,c))$.
\item[(L3) Commutativity.] $a\perp b\implies\seqop(a,b)=\seqop(b,a)=\parop(a,b)$.
\item[(L4) Branch distribution.] $\seqop(\ifop(p,a,b),c)=\ifop(p,\seqop(a,c),\seqop(b,c))$.
\item[(L5) Fold/unfold.] $\callop(\text{id})=\operatorname{frame}_{F}(b(F))$, where $F=\mathcal{R}[\text{id}]$. Unfolding retains $F$'s input/output boundary and scope.
\end{description}

\paragraph{Soundness.}
\label{sec:soundness}
Three honesty conditions (each just restating a rule the normal forms already enforce) make the model exact. \textbf{(A1) Honest footprint:} a leaf's output record $\delta_\ell(\sigma)$ is fixed by the fields it declares to read plus its own effect outcome, and it never reads an undeclared field: this is what a truthful FNF-2 schema means. \textbf{(A2) Outcomes follow the leaf:} the result of a nondeterministic effect (say, an LLM call) is fixed once as a value (a single \emph{draw} from its distribution) that belongs to the leaf's identity, not to its position in the running order; reordering the leaf does not change its value. With every such outcome fixed, each body $\sem{f}$ is an ordinary deterministic function, so the laws hold for the fixed outcomes and therefore also in distribution. \textbf{(A3) Pure predicates:} a branch test reads only stored fields ($\Comp(p)$, Section~\ref{sec:architectural-levels}); a decision that needs the model is first computed by an atomic leaf (FNF-1) and stored, then read purely.
\begin{lemma}[Framing]\label{lem:framing}
If $\operatorname{Write}(a)\cap\operatorname{Read}(b)=\emptyset$, then running $a$ before $b$ changes neither the fields nor the external resources $b$ reads. With the leaf outcomes fixed as in (A2), $\delta_b(\sem{a}(\sigma))=\delta_b(\sigma)$: the declared inputs and external observations on which $b$ depends are unchanged.
\end{lemma}
\begin{theorem}[Soundness]\label{thm:soundness}
Under (A1), (A2), (A3), and truthful resource-access contracts, each of L1 to L5, applied to a body in FNF-2, preserves the record transformer and observable external behavior, and gives a body in FNF-2.
\end{theorem}
\begin{proof}
Each law is checked by unfolding the semantics on an arbitrary state $\sigma$, using two facts from the definitions: a sequence composes its steps, $\sem{\seqop(a,b)}=\sem{b}\circ\sem{a}$, and running any body merges its output record into the state, $\sem{x}(\sigma)=\sigma\uplus\delta_x(\sigma)$.

\emph{L1 (identity).} $\sem{\seqop(\nopsym,f)}(\sigma)=\sem{f}(\sem{\nopsym}(\sigma))=\sem{f}(\sigma)$ since $\sem{\nopsym}$ leaves the state unchanged, and symmetrically for $\seqop(f,\nopsym)$.

\emph{L2 (associativity).} The two groupings $\seqop(\seqop(a,b),c)$ and $\seqop(a,\seqop(b,c))$ denote the same transformer: since $\sem{\seqop}$ is function composition and composition is associative, each evaluates on every state $\sigma$ to $\sem{c}(\sem{b}(\sem{a}(\sigma)))$.

\emph{L5 (fold/unfold).} Both sides run the same registered body on $\sigma|_{\Sin(F)}$ and merge only its declared outputs $\Sout(F)\cup\Sopt(F)$ into the caller. The frame and resource/guardrail context are retained, so local fields do not leak. Combined with L2 and a checked extraction interface, this gives extract/inline.

\emph{L3 (commutativity).} Let $a\perp b$. Unfolding the sequence step by step,
\[
\begin{aligned}
\sem{\seqop(a,b)}(\sigma)&=\sem{a}(\sigma)\uplus\delta_b(\sem{a}(\sigma))\\
&=\sigma\uplus\delta_a(\sigma)\uplus\delta_b(\sem{a}(\sigma)).
\end{aligned}
\]
By Lemma~\ref{lem:framing} $b$ reads none of $a$'s writes, so $\delta_b(\sem{a}(\sigma))=\delta_b(\sigma)$; symmetrically $\delta_a(\sem{b}(\sigma))=\delta_a(\sigma)$. Hence
\[
\begin{aligned}
\sem{\seqop(a,b)}(\sigma)&=\sigma\uplus\delta_a(\sigma)\uplus\delta_b(\sigma),\\
\sem{\seqop(b,a)}(\sigma)&=\sigma\uplus\delta_b(\sigma)\uplus\delta_a(\sigma).
\end{aligned}
\]
Independence also gives $\Sout(a)\cap\Sout(b)=\emptyset$, so the two records write disjoint fields and $\uplus$ merges them to the same state in either order, which is exactly $\sem{\parop(a,b)}(\sigma)$. The external-resource conditions exclude conflicting effects as well, so the swap preserves their externally visible result.

\emph{L4 (branch distribution).} The predicate $p$ is state-pure (A3), so it takes the same value on both sides at $\sigma$: when $p(\sigma)$ holds both reduce to $\sem{c}(\sem{a}(\sigma))$, otherwise both reduce to $\sem{c}(\sem{b}(\sigma))$.

\emph{FNF-2 preservation.} L1 and L2 change neither leaf schemas nor execution order. L5 preserves the call frame and the same declared inputs and outputs, so read availability is unchanged. L4 puts $c$ in each branch arm, where it sees at least the post-join schema. L3 swaps independent steps: neither reads the other's writes, and their writes are disjoint, so both remain well-typed and yield the same state. A parallel schedule of an antichain follows by repeated such swaps.
\end{proof}

\paragraph{Applications.} As directed rewrites the laws are the compiler's optimizations (Table~\ref{tab:algebra-laws}), and they ground FNF-3: L2 groups a contiguous sub-workflow and L5 names it $\callop(\text{id})$ (extract/inline), while L4 factors a shared post-branch step through a branch. The laws do not choose what to factor; candidates come from the FNF-3 search. They only supply the meaning-preserving rewrite that performs the factoring once a candidate is accepted. FNF-3 is thus normalization \emph{driven by} the laws: it decides \emph{when} to apply these moves and fixes the one canonical identity. For exact copies the moves preserve behavior (Theorem~\ref{thm:soundness}); near copies differ, so they are factored only after the architect judges that they do the same job. L3 is the compiler's parallelization rule: independent steps, read off the dependence graph, run concurrently to cut wall-clock latency. This and further schema-driven rewrites (interrupt fusion) are listed with +O. Appendix~\ref{app:laws} applies each law to the running example.

\begin{table*}[t]
\centering
\caption{Algebraic laws and compiler optimizations.}
\label{tab:algebra-laws}
\small
\begin{tabularx}{\linewidth}{@{}l l X@{}}
\toprule
Law & Optimization & Effect \\
\midrule
L1 & Empty-step elimination & Insert or remove only an effect-free do-nothing step; logging, metrics, and checkpoints are not identities. \\
L2 & Sub-workflow extraction & Regroup a sub-workflow and register it as a FAOP under FNF-3. \\
L3 & Auto-parallelization & Schedule independent steps with $\parop$, reducing wall-clock time. \\
L4 & Shared-step factoring & Read right to left, write a step that ends both arms once, after the branch (FNF-3). \\
L5 & Canonical reuse & Replace a framed inline body by its registered call, preserving declared inputs, outputs, and scope (FNF-3). \\
\bottomrule
\end{tabularx}
\end{table*}

\subsection{Extensions}
\label{sec:extensions}

A FAOP that satisfies all four normal forms is structurally well-formed: its leaves are explicit (FNF-1), its data flows are sound (FNF-2), its shared sub-workflows are canonical (FNF-3), and its resource declarations match their recursively expanded bodies while local guardrails are propagated per route (FNF-4). Structural well-formedness is necessary for safe composition, but it is not sufficient for all engineering concerns.

Extensions are denoted by a \texttt{+} prefix and are recorded in the FAOP's signature (Section~\ref{sec:maturity-card}) alongside the normal-form level. The set of extensions presented here is not closed; defining additional extensions for domain-specific concerns is ongoing work. Appendix~\ref{app:extensions} applies all three to the running example: parallel ledger reads (+O), stopping limits with latency and cost bounds (+B), and reachability obligations (+R).

\subsubsection{+O Optimized}

Extension +O records application of a specified set of semantics-preserving rewrites that reduce latency or user round-trips. It does not certify a globally cheapest workflow. Examples follow; extending the rewrite set is ongoing research.

\emph{Parallelization.} For adjacent siblings $a,b$ in a $\seqop$ node, if $a\perp b$ (Definition~\ref{def:independence}), then $\seqop(a,b)$ is rewritten to $\parop(a,b)$ so the two run concurrently, reducing wall-clock latency.

\emph{Interrupt fusion.} If adjacent interrupts $\{\iota(q_1,k_1),\ldots,\iota(q_m,k_m)\}$ are pairwise independent and have disjoint query keys, they may be fused into a compound interrupt $\iota(\{q_1,\ldots,q_m\},k)$. No question or its applicability may depend on another answer, and the combined continuation must preserve the original answer handling. This reduces $m$ user round-trips to one.

\emph{Unused steps} are not rewritten by +O: +R reports them to the architect (Section~\ref{sec:ext-r}).

\emph{Prerequisites:} FNF-2 schemas and declared external-resource accesses. Parallelization and interrupt fusion use the full independence test, not disjoint answer or output keys alone.

\subsubsection{+B Bounded}
\label{sec:ext-b}

A service may never reply, a person may never answer, and a model-driven loop may never reach its intended exit. Extension +B requires finite stopping limits and time/cost bounds for every execution path, so the executor finishes or reports failure within a declared budget. It does not promise successful completion.

\paragraph{Definition (+B).} A FAOP $F$ satisfies +B iff every loop, including loops in callees and handlers, has a finite maximum iteration count $N_{\max}$; every action, interrupt, and checkpoint has an enforceable latency/cost pair $(\tau_{\max},c_{\max})$; and all remaining work, including predicates, continuations, dispatch, cancellation, and failure reporting, has a finite accounted bound. An interrupt has a response deadline, not an assumption that a person eventually responds. Every timeout and exhausted loop must lead to bounded handling or explicit failure, never an implicit unbounded retry.

\paragraph{Composition of bounds.} The compiler works bottom-up over the acyclic call graph. A sequence sums its children's latency and cost; a branch takes the componentwise maximum over its arms; and a parallel group takes the maximum latency and sums costs. Each includes its own bounded overhead. A loop includes at most $N_{\max}$ body invocations and predicate evaluations, plus bounded exhaustion handling. A call includes its callee's bound and frame/dispatch overhead. Interrupt continuations and all timeout/failure handlers are bounded in the same way, including any calls or retries they contain.

\emph{Prerequisites:} typed bodies and an acyclic call graph. +B does not require +R: a loop may be capped without proving that its branches are reachable. Bounds must be enforced by a conforming executor and, where needed, by service-side limits or reservations. A local timeout alone does not bound the cost of a remote operation that keeps running.

\begin{proposition}[Bounded completion or failure]
If $F$ satisfies +B and the declared enforcement contracts hold, every run finishes or reports failure within the computed latency and cost bound.
\end{proposition}
\begin{proof}
Every primitive and control overhead has a finite bound. The call graph is acyclic, each loop has finitely many iterations, and every continuation or handler is subject to the same rules. Induction over these constructs gives a finite bound on every success and failure path. This bounds the executor's outcome, not the success of the requested business task.
\end{proof}

\subsubsection{+R Reachable}
\label{sec:ext-r}

A branch may be well-typed yet unreachable because of what earlier steps establish. A step may also run on every path yet compute a result that nothing uses. Extension +R rules out both, in two senses of reaching: control reaches every branch arm, and every computed result reaches a use.

\paragraph{Definition (+R).} A FAOP $F$ satisfies +R when (i) both arms of every branch in its body, and in the FAOPs it calls, are proved reachable from its declared entry domain, and (ii) every \emph{pure} step in these bodies, one built only from value leaves, directly or through the FAOPs it calls, writes at least one field that a later step or predicate may read before another step overwrites it, or that the FAOP returns. A step with an action, interrupt, or checkpoint is never unused, since its effect is observable even if no field it writes is read. The proof of (i) is relative to the declared transition contracts; it says nothing about how long a step takes or whether a service ever replies.

\paragraph{How it is checked.} The compiler checks (ii) itself with liveness analysis, the standard dataflow analysis that runs from the end of a body toward its start and finds, at each point, the fields that a later step may still read~\citep{kildall1973}. Under (A1) every read is declared, so the result is exact for the declared schemas and needs no contracts. A step whose only readers lie in an arm found unreachable is reported as unused too.

For (i), FAOP adds no checker of its own: it uses bounded model checking~\citep{biere1999bmc}, which reduces the question whether a run of at most $k$ steps reaches a given state to propositional satisfiability, so that a satisfying assignment is a witness run. Bounded model checking works on a finite boolean encoding of the state, so each field is encoded in a finite domain. The mapping is direct: each field becomes a state variable; each leaf may return any value that its output contract allows, or any value of its type if it has none; each branch keeps its condition; and each call runs in its frame. Each arm then becomes one question, and a witness proves it reachable. With the loop caps of +B, and since the call graph is acyclic, every run has at most a known number of steps; taking $k$ as that number covers every run, so finding no witness shows that no run consistent with the contracts and the encoding reaches the arm. If the checker gives no answer within its limits, the result is ``not proved'' and +R is not awarded. Testing the condition alone is not enough: suppose, for example, that the step computing \texttt{balance} never returns a negative value; then the arm under \texttt{balance < 0} never runs, although the test itself can be true. Both kinds of finding go to the architect, who removes the arm or the step, or corrects the condition, contract, or schema behind it; the compiler removes nothing on its own.

\paragraph{The extension set is open.} New extensions can be defined for domain-specific guarantees without altering the core or the four normal forms. Two important examples are \emph{transactional atomicity} (+T), where an agent that moves money, sends messages, or writes to external systems needs an all-or-nothing commit so that a partial failure never leaves the outside world half-updated, and \emph{compile-time state contracts} (+C), which check that the values passing between steps satisfy declared pre/post conditions and invariants before deployment. Both matter for real-world AI agents, and we plan to develop them in the near future (Section~\ref{sec:future-work}).

\subsection{AI Architectural Levels}
\label{sec:architectural-levels}

When an SOP is encoded as an AOP, a common shortcut is to hand the whole AOP to a ReAct-style agent that drives every step with the model in the loop. Most of those steps are fixed database and API calls, so the model adds cost, latency, and lost reproducibility to work that needs none of it, and a mostly deterministic procedure inherits the governance reserved for autonomous systems. The real question is which steps need a model, and an unnormalized AOP cannot answer it: the requirement is hidden inside opaque prose, so the choice is all or nothing and governance applies to the whole artifact. Normalization resolves this: FNF-1 exposes every step as a separate, effect-labeled leaf and FNF-2 gives each a typed contract, so the need for a model is decided per leaf and per routing predicate. FAOP thus turns the decision into a classification computed from the normalized body.

The classification rests on two tests, both decided by static inspection of the typed body. A leaf is deterministic, written $\Det(\ell)$, when its effect label is neither $\opLLM$ nor $\opPlan$: the step is a tool, DB, API, file, queue, or value computation ($\opCompute$) rather than a model call or a planning step. A predicate is computable, written $\Comp(p)$, when the branch or loop condition is an ordinary typed boolean over fields already in the state, for example \texttt{balance < 0}, rather than a judgment delegated to the model.

Planning is itself a valid FAOP action: a planning leaf carries the effect label $\opPlan$ from the effect signature (Section~\ref{sec:fnf1}) and produces a FAOP either by selecting and composing existing registered FAOPs, or by generating a new body. A planner is therefore an ordinary FAOP whose body contains a planning leaf.

\begin{definition}[Architectural level]
Writing $\predsop(F)$ for the $\ifop$ and $\loopop$ predicates of $F$, the level $\lambda(F)$ is defined as follows.
\begin{itemize}
\item Level 0 (no AI): $\forall\ell\in\leavesop(F):\Det(\ell)$ and $\forall p\in\predsop(F):\Comp(p)$. No model is invoked anywhere.
\item Level 1 (AI agent): some leaf is non-deterministic ($\opLLM$), but every branch and loop condition is still a computable boolean ($\Comp$); the model may produce a value a condition later tests, but it does not select or generate which FAOPs run.
\item Level 2 (agentic AI, selector planner): $F$ contains a planning leaf that selects and composes only existing registered FAOPs; no unverified logic enters the system.
\item Level 3 (agentic AI, generator planner): $F$ contains a planning leaf that generates a new body; the produced FAOP may execute only after it passes FNF-1 through FNF-4.
\end{itemize}
\end{definition}

Because a planner emits FAOPs that must themselves be normalized, the levels stack: a Level-2 or Level-3 FAOP produces Level-0 or Level-1 bodies. We treat the planning leaf as atomic and assume each emitted body is normalized before it runs; the reflective case, where bodies are normalized at runtime, is left to future work (Section~\ref{sec:future-work}).

The running example is Level~1: one leaf calls a model to summarize account activity, while every other leaf and predicate is deterministic (Appendix~\ref{app:need-ai}).

Proving a lower level is a benefit in itself: every leaf shown deterministic bypasses the model, so a body with only a few $\opLLM$ leaves runs mostly with no token cost, no model latency, and full reproducibility. Recorded in the FAOP's signature (Section~\ref{sec:maturity-card}), the level informs model-risk review by distinguishing a non-model workflow, an LLM-assisted tool, and a runtime planner. It does not determine legal obligations; governance also depends on the use case, data, and organizational policy.

\paragraph{Practical distribution of levels.} We expect the vast majority of production workflows, particularly in regulated domains such as finance, to lie at Levels 0 and 1, where the body is fixed and auditable and execution must follow the specification exactly. Level 2 selectors are plausible but uncommon (at most a single organization-wide planner that assembles workflows from the verified registry), and Level 3 generators, which synthesize new bodies at runtime, are rare to non-existent in settings where no execution error is tolerated.

\subsection{Signature and Maturity Card}
\label{sec:maturity-card}

Every compiled FAOP carries a \emph{signature}: a compact, compiler-computed summary of what that FAOP has been shown to satisfy, written as the triple $(k,E,\lambda)$:
\begin{itemize}
\item $k\in\{0,1,2,3,4\}$ is the highest normal form the FAOP satisfies; $k=3$, for instance, means it meets FNF-1 through FNF-3 but not FNF-4.
\item $E$ is the set of extensions it satisfies, for example $\{+O,+B,+R\}$ (Section~\ref{sec:extensions}).
\item $\lambda\in\{0,1,2,3\}$ is its AI architectural level (Section~\ref{sec:architectural-levels}): whether the FAOP uses no model, a model as a tool, or a runtime planner.
\end{itemize}

Under the stated proof and runtime-enforcement assumptions, the running example has signature $(4,\{+O,+B,+R\},1)$ (Appendix~\ref{app:signature}).

The signature makes per-FAOP governance a check on the triple: an organization writes its deployment policy as a rule over $(k,E,\lambda)$: development might allow any $k\geq 2$, while production might require $k=4$, $+B\in E$, and $\lambda\leq 1$, reserving the strongest model-risk controls for planners ($\lambda\in\{2,3\}$).

The registry aggregates the signatures of all its FAOPs into a single \emph{Maturity Card}: a portfolio view of the whole library (the distribution of normal-form levels $k$, which extensions are satisfied where, and the spread of architectural levels $\lambda$) that turns library-wide compliance into a schema-level, checkable property.

\subsection{Summary of Benefits}

Table~\ref{tab:benefits} collects the consequences of the four normal forms and their extensions for an AOP library. Each row names a benefit, the behavior under FAOP that delivers it, and the responsible FAOP mechanism.

\begin{table*}[t]
\centering
\caption{Benefits delivered by FAOP normalization. Source identifies the FNF level or extension responsible.}
\label{tab:benefits}
\scriptsize
\begin{tabularx}{\linewidth}{@{}r l X l@{}}
\toprule
\# & Benefit & Effect under FAOP & Source \\
\midrule
1 & Integrity across workflows & Affected workflows are notified and reviewed again; the compiler rejects structural and interface incompatibilities. & Registry \\
2 & Scope validation & Resource footprints and per-route guardrail declarations are checked; runtime confinement requires an enforcing executor. & FNF-4 \\
3 & Type safety & Compiler verifies $\Sin(x)\subseteq S$ for every step and predicate $x$ before deployment. & FNF-2 \\
4 & Auto-parallelization & Independence proves disjoint dependencies; siblings are scheduled concurrently. & FNF-2, +O \\
5 & No unused steps & Pure steps whose results are never read or returned are reported to the architect; effectful steps are never unused. & +R \\
6 & Interrupt batching & Independent interrupts with disjoint keys merge into one form. & +O \\
7 & LLM bypass & Effect-labeled non-LLM leaves dispatch directly; only $\opLLM$ leaves invoke the model. & FNF-1 \\
8 & Reuse without duplication & Sub-workflow extracted into one registered FAOP, called everywhere, unless the architect declares the copies distinct. & FNF-3 \\
9 & Bounded completion & Loops, waits, handlers, and execution overhead are bounded; the run finishes or reports failure. & +B \\
10 & Reachability check & Bounded model checking finds a witness run for every branch arm under the declared contracts; arms without one are not certified. & +R \\
11 & Observability & Each leaf carries identity, typed schemas, and an effect label. & FNF-1 \\
12 & Partial retry & A run can resume from its last $\chi$ checkpoint. & Checkpoint \\
13 & Isolated testing & Each leaf is testable in isolation against $\Sin,\Sout$. & FNF-1 \\
14 & Safe refactoring & Algebraic laws guarantee semantic preservation. & Algebra \\
15 & Impact analysis & The registry identifies all direct and transitive dependents for notification and renewed review. & FNF-3 \\
16 & Progressive adoption & Four normal forms admit incremental migration. & FNF-1--4 \\
17 & Architectural-level classification & $\lambda\in\{0,1,2,3\}$ is computed from the typed body and informs governance, not legal obligations. & Section~\ref{sec:architectural-levels} \\
18 & No hidden recursion & The call graph is required acyclic (Convention~\ref{conv:call-graph}); any recursive cycle is rejected at compile time. & FNF-3 \\
\bottomrule
\end{tabularx}
\end{table*}

\section{Toolchain: SkillBase and the Executor}
\label{sec:toolchain}

A theory is only useful if its checks can be enforced on the artifacts an organization ships. The FAOP toolchain does this in two top-level pieces, shown in Figure~\ref{fig:toolchain}: \textbf{SkillBase}, the design-and-store side that checks and keeps FAOPs, and a separate \textbf{executor} that runs them. SkillBase plays the role of PostgreSQL for an AOP library: it supplies the schema, constraints, and queryable metadata that turn a folder of files into a governed asset. In one line, SkillBase is an \emph{AOP management system}. SkillBase has three components; the executor is the fourth part of the toolchain.

\begin{itemize}
\item \textbf{Vocabulary}: the controlled list of canonical field, resource, and guardrail names; a resource names any external dependency a leaf uses, such as a tool, agent, database, API, or queue. Every name in a FAOP must resolve here, so one workflow cannot call a field \texttt{account\_id} while another calls the same field \texttt{acct\_no}.
\item \textbf{Compiler}: turns a source artifact into a checked, normalized FAOP. It parses the body, resolves names against the vocabulary, and runs the FNF-1--FNF-4 passes and the algebra and extension checks of Section~\ref{sec:normal-forms} \emph{against the architect's declarations} (effect labels, scopes, and loop limits), flagging violations rather than inferring intent. It then computes the signature $(k,E,\lambda)$ and emits the optimal, normalized blueprint consumed along either path below.
\item \textbf{Registry}: the index of record over the FAOP source artifacts: versioned FAOP bodies, the call DAG, a reverse-dependency index that answers impact queries when a shared FAOP changes, and the per-FAOP signatures the Maturity Card aggregates. The artifacts may be stored in git or a database, and the current implementation supports both; the stored artifacts are the source of truth and the registry is a queryable index over them. A shared-procedure change triggers notification of every affected workflow and its owners, following direct and transitive call edges. Each affected workflow must be reviewed again before adopting the updated blueprint. The compiler detects structural and interface incompatibilities; unchanged schemas do not establish unchanged business meaning, which still needs review and tests.
\item \textbf{FAOP executor}: a separate runtime that runs a blueprint; it is the component that turns the compiled normal-form and extension guarantees into enforced runtime behavior. It dispatches each leaf by its effect symbol to the concrete tool, API, database, or model the leaf names, admits that call only when the scope $\Sigma$ permits it, handles interrupts and checkpoints, runs each FAOP invocation in its own state scope so that only its declared outputs cross back to the caller, and runs proven-independent steps in parallel. Because it consumes only the blueprint, the executor can be built on an existing workflow engine rather than reimplementing one. The tools, APIs, databases, and models themselves lie outside the toolchain; the executor is the boundary that gates every call against $\Sigma$.
\end{itemize}

\paragraph{A spectrum of executors.} The theory fixes what a FAOP means, not how a runtime executes it, so the executor admits a range of implementations, from a generic agent handed the whole FAOP (trivial to build but enforcing nothing, the ReAct shortcut of Section~\ref{sec:architectural-levels}) to a runtime that interprets the body structurally. Ours does the latter.

\paragraph{Execution model.} The executor runs a FAOP by realizing the denotational semantics of Section~\ref{sec:algebra} directly: it walks the body, threading the state record $\sigma$ through each construct. A $\seqop$ runs its steps in order; an $\ifop$ evaluates its predicate and takes the matching arm; a $\loopop$ repeats its body until the predicate holds or a declared +B cap $N_{\max}$ invokes bounded exhaustion handling; and a $\callop(\text{id})$ opens a fresh frame for $\mathcal{R}[\text{id}]$, passing in only $\Sin$ and merging back only its declared outputs, so the called FAOP runs in its own state scope, exactly as a function call in a conventional language runs in its own frame, seeing only its arguments and returning only its result. The implementation also lets a call rename the fields it passes in and returns, and fix some inputs to constants, so one shared FAOP can serve callers whose fields have different names. Each leaf is dispatched by its effect label to the bound tool, model, database, or in-process function, and admitted only if the scope $\Sigma$ permits it; proven-independent siblings run concurrently, with their disjoint output records merged by $\uplus$. An $\iota$ interrupt suspends the run with its query and resumes the continuation when the answer arrives, and a $\chi$ checkpoint persists the running state so the workflow can resume from that point.

\paragraph{Two execution paths.} Compilation produces a single optimal, self-contained \emph{blueprint}: the full FAOP (typed schemas, effect labels, scope, extensions, the normalized call DAG, every FAOP it transitively calls, and the vocabulary) stored in the registry alongside the discovery metadata $(n,d,\Sin,\Sout)$ and the signature, so a consumer needs nothing else to run it. The same blueprint feeds both paths. A FAOP-aware executor consumes it directly, materializing it into a runnable graph and enforcing every normal-form guarantee and declared extension at runtime. For any other executor (a runtime that is not FAOP-aware), an \emph{adapter} converts the same blueprint into an AOP in that executor's own format, so an existing runtime can run the normalized workflow unchanged. The adapted AOP keeps the design-time benefits (explicit typed steps, canonical composites, and a clean description for discovery) and keeps the runtime guarantees, chiefly scope confinement (FNF-4) and effect-based LLM bypass, to the extent that the target executor can enforce them. The blueprint is thus the single compiled artifact: the FAOP-aware path runs it with full guarantees, while the adapter path may trade some of them for backward compatibility. Appendix~\ref{app:exporting} exports the running example both ways.

\paragraph{FAOP as a language.} FAOP defines a small domain-specific language, playing the role for agent workflows that SQL plays for relational data: a deliberately restricted language that a system can check and optimize. The grammar of Section~\ref{sec:composite} is its syntax; the effect, type, and scope checks of FNF-1, FNF-2, and FNF-4 (with +B and +R) are its static semantics; the algebra of Section~\ref{sec:algebra} is its equational theory; SkillBase's compiler checks and normalizes its programs; and the FAOP executor is its reference runtime.

\paragraph{The Skill Architect.} The compiler does not do everything on its own. The Skill Architect (the DBA of the agentic era, Section~\ref{sec:formal-framework}) supplies the declarations the compiler checks and makes the judgments the rules leave open: authoring or reviewing the source artifact; declaring each leaf's effect label (FNF-1) and scope (FNF-4) and any loop limit (+B), which the compiler verifies but does not infer; deciding the FNF-3 candidate pairs; curating vocabulary terms; and setting the deployment policy over FAOP signatures. The compiler enforces the mechanical rules; the architect makes the decisions the rules cannot. We describe the Skill Architect as a person, but the role is not intrinsically human: the same steps (drafting a source artifact, lifting it into the grammar, and proposing effect labels, scopes, and loop limits for the compiler to check) can be carried out by an AI agent, the generator case of Level~3 (Section~\ref{sec:architectural-levels}) and the direction the auto-authoring skill literature already points to \citep{wang2023voyager,li2024autoflow,zhang2025aflow,zhou2026mementoskills}. This paper keeps the process deterministic and architect-reviewed; automating more of it is future work (Section~\ref{sec:future-work}).

\begin{figure*}[!ht]
\centering
\includegraphics[width=\textwidth]{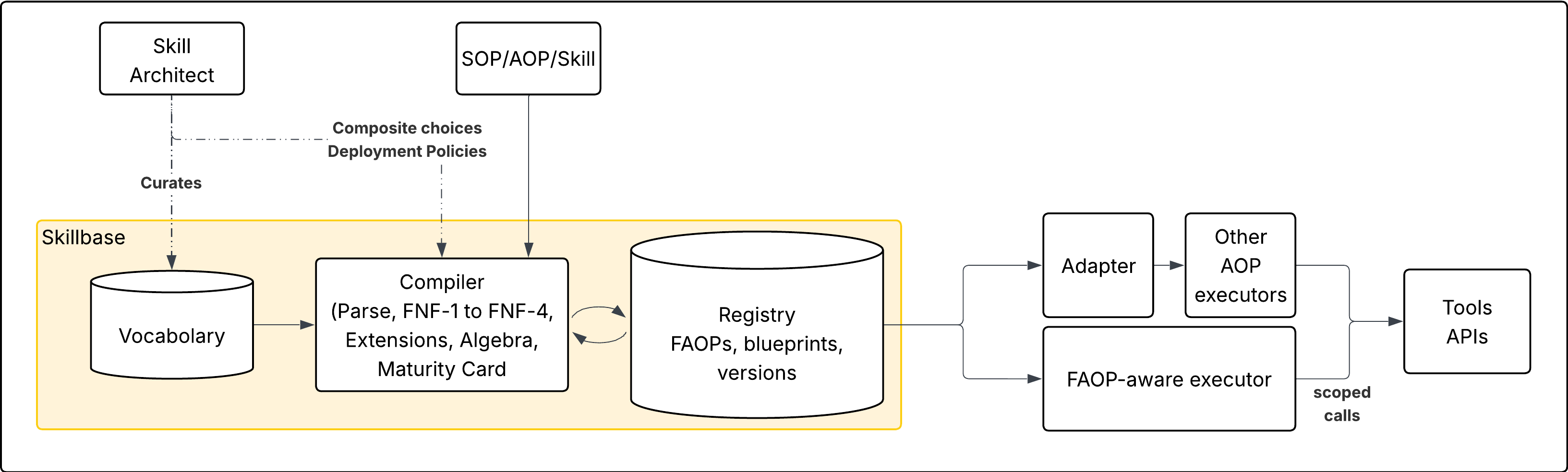}
\caption{The FAOP toolchain. A Skill Architect, human or AI agent, authors the source artifact and makes the design decisions that the compiler then verifies. SkillBase (vocabulary, compiler, registry) checks the artifact and stores each compiled blueprint with its description and signature. The FAOP executor runs the blueprint and enforces scope $\Sigma$ on every call; for other executors, an adapter converts the blueprint into their own AOP format, keeping the structural benefits and the runtime guarantees to the extent that executor can enforce them.}
\label{fig:toolchain}
\end{figure*}

\section{Evaluation}
\label{sec:evaluation}

This paper is intentionally theory-first: it establishes the formalism, the four normal forms, the FAOP algebra, and the toolchain, and defers a full empirical study. The worked normalization of Appendix~\ref{app:normalization-examples} serves as the end-to-end demonstration.

\section{Related Work}
\label{sec:related-work}

\subsection{Skill and Tool Frameworks}
\label{sec:skill-frameworks}

ReAct \citep{yao2023react} established the pattern of LLM agents interleaving reasoning with tool calls, and HuggingGPT \citep{shen2023hugginggpt} extended it by selecting tools from a registry via semantic description matching. More recent systems standardize the skill and tool interface itself: OctoTools \citep{lu2025octotools} formalizes tool cards (name, description, I/O schema, and constraints) with a planner--executor architecture that outperforms AutoGen, GPT-Functions, and LangChain by up to 10.6\% on tool-use benchmarks, while orchestration frameworks such as Semantic Kernel \citep{microsoft2023semantickernel}, LangChain \citep{chase2023langchain}, AutoGen \citep{wu2024autogen}, CrewAI~\citep{crewai2024}, Google A2A~\citep{google2024a2a}, and Anthropic MCP \citep{anthropic2024mcp} converge on a common \texttt{\{name, description, inputSchema, outputSchema\}} contract. Around these frameworks the skill abstraction has grown into a studied infrastructure layer in its own right, surveyed across its lifecycle (representation, acquisition, retrieval, and evolution \citep{zhou2026skillsurvey}) and extended beyond text to multimodal \emph{visual} skills \citep{xu2026visualskills}.

Yet most of this work defines, standardizes, and measures the \emph{atoms} of agent systems, not a theory for reshaping the \emph{library} they form. FAOP supplies that theory: it treats each skill as a typed function within a formal algebra, applies progressive normal forms for incremental guarantees, and provides a compiler that catches composition errors before deployment. It is also modality-agnostic, since it normalizes a workflow's body regardless of what each leaf computes.

Closest in spirit is a lifecycle survey of \emph{evolving} skill libraries~\citep{li2026dynamicskills}. Its evidence points the same way FAOP does: admission matters, maintenance becomes load-bearing as libraries grow, libraries that accept every proposed skill accumulate redundant and conflicting entries, and flat retrieval often degrades at moderate library sizes. That survey asks how libraries change; FAOP supplies the structure that makes a change checkable: typed steps, one definition per shared sub-workflow, declared scope, and a registry that names every affected workflow.

A close neighbor is GraSP \citep{xia2026grasp}, which inserts a compilation layer between skill retrieval and execution: it compiles the skills \emph{retrieved for a task} into a typed DAG with precondition--effect edges, verifies nodes at runtime, and repairs failures through a five-operator local-repair algebra. GraSP and FAOP address different problems: GraSP compiles a single task's retrieved skills for runtime orchestration performance, its algebra patching a failing execution graph in place; FAOP normalizes the \emph{whole library} at design time: progressive normal forms, an equational algebra of \emph{semantics-preserving rewrites}, compile-time scope isolation, and registry-level impact analysis, none of which GraSP provides. Their typing is where they meet: GraSP's precondition--effect edges are the runtime counterpart of FNF-2, so a GraSP-style executor could run a workflow FAOP has already normalized and scoped.

Another close neighbor is Agentics~2.0 \citep{gliozzo2026agentics}, which models each LLM call as a typed function over records, composes such functions under identity, associative composition, and Map-Reduce, and validates their outputs at run time. FAOP shares this typed-function view, but its algebra adds commutativity, branch distribution, and fold/unfold, and it normalizes, scopes, and checks a whole library before deployment rather than one program at run time.

\subsection{Prompt Engineering and Structured Reasoning}
\label{sec:prompt-engineering}

The foundation of modern agent workflows rests on prompt engineering advances. Wei et al. \citep{wei2022cot} showed that decomposing reasoning into intermediate steps dramatically improves LLM performance. Yao et al. \citep{yao2023react} extended this by interleaving reasoning with tool-calling actions, creating the reason-act-observe loop that underlies most agent executors today. Schulhoff et al. \citep{schulhoff2024promptreport} list 58 distinct prompting techniques, and Zhang et al. \citep{zhang2025cotsurvey} trace the evolution from single-prompt reasoning to multi-step agentic systems.

A skill or AOP is, at its core, a prompt: a step-by-step operational prompt that guides an agent through a workflow. Prompt engineering techniques such as chain-of-thought \citep{wei2022cot} and few-shot prompting \citep{brown2020fewshot} improve the quality of each individual step. FAOP still uses prompt engineering at the step level; it does not replace it. What FAOP adds is the theory for managing large skill libraries, the problem that arises when an organization has hundreds of skills sharing common sub-workflows. Specifically, FAOP provides integrity across skills, scope isolation, and compile-time optimizations.

\paragraph{Workflow generation and extraction.} A separate line of work produces source workflows from natural-language inputs. AutoFlow \citep{li2024autoflow} generates agent workflows as natural-language programs from task descriptions, refined by fine-tuning or in-context optimization. AFlow \citep{zhang2025aflow} searches the space of code-represented agentic workflows via Monte-Carlo Tree Search with execution feedback, achieving a 5.7\% average improvement over strong baselines on six benchmarks and enabling smaller models to match or outperform GPT-4o on specific tasks at a fraction of the inference cost. Flow-of-Action \citep{pei2025flowofaction} structures Site-Reliability-Engineering expertise as explicit SOPs, and synthesizes new ones when none exist, then constrains a multi-agent system to follow them, lifting Root-Cause-Analysis accuracy from 35.5\% under ReAct to 64.0\%. These pipelines produce the initial artifact (the natural-language or code-shaped workflow) but stop there: they do not typecheck it, do not deduplicate it against a shared registry, and do not confine its scope. FAOP consumes such an artifact as $\mathcal{R}_0$ (Section~\ref{sec:fnf1}) and then runs the FNF-1--FNF-4 passes over it, so the generator supplies the raw material and FAOP supplies the normalization.

\subsection{Formal Verification and Safety}
\label{sec:formal-verification}

To our knowledge, no prior work directly addresses the problem FAOP solves: normalizing skill libraries with progressive normal forms and a FAOP algebra. Recent systematizations of the skill layer in fact call for exactly this (\emph{verifiable} and \emph{certifiable} skills \citep{badhe2026survey,jiang2026sok}), naming the gap FAOP fills. Several recent works, grouped below, tackle adjacent problems.

\paragraph{Scope and permission enforcement.} Odersky et al.~\citep{odersky2026capabilities} use capability tracking in the type system to enforce tool-access restrictions at compile time. Progent~\citep{shi2025progent} enforces per-step permissions through a runtime policy engine; MiniScope~\citep{zhu2025miniscope} reconstructs least-privilege permission hierarchies; Prompt Flow Integrity~\citep{kim2025pfi} adds agent isolation and guardrails against privilege escalation; and a policy algebra~\citep{tripathi2026policyalgebra} composes profiles, tools, and budgets so authority narrows across delegation. FNF-4 instead checks resource-footprint declarations and per-route guardrail provenance. It does not impose parent-to-child authority monotonicity or replace authorization and runtime enforcement. The capability-based governance called for by skill-security work~\citep{xu2026skillarch} remains a separate enforcement concern to which these checked declarations can contribute.

\paragraph{Workflow verification.} $\lambda_A$ \citep{liu2026lambdaa} defines a typed lambda calculus for agent composition, whose pipeline composition satisfies identity and associativity, with Coq-mechanized type-safety and termination proofs; its audit of 835 real GitHub agent configurations found 94.1\% structurally incomplete, confirming that these defects are endemic rather than rare. Agentproof \citep{xavier2026agentproof} extracts a graph model from four agent frameworks and checks structural properties of that graph, such as whether every exit can be reached from the entry and whether any node is a dead end, together with temporal policies over its paths, and Lean4Agent \citep{wang2026lean4agent} uses Lean 4 dependent types for full workflow verification, reporting about 12\% higher task success for verified workflows. GraphFlow \citep{morris2026graphflow} requires the pre- and postcondition contracts of acyclic, sequential workflow diagrams to be proof-checked before admission to a shared library, a verification layer its authors describe as specified and in active development. FAOP's core contribution is library normalization and checked rewrites; +R maps its reachability questions to bounded model checking, while the structural checks run on the specification before code exists.

\paragraph{Policy and contract compilation.} Declarative Policy Compilation \citep{chen2026policy} compiles one DSL specification into multiple deployment artifacts with compile-time verification: close to FAOP's pipeline, but without normal forms or algebraic optimization. AgentSpec \citep{wang2026agentspec} enforces runtime safety constraints structured as triggers, predicates, and enforcement mechanisms, and Agent Behavioral Contracts \citep{bhardwaj2026contracts} give four-tuple contracts $(P,I,G,R)$ with sufficient conditions for safe composition. FAOP formalizes the same compositionality question as equational laws over workflow bodies and, unlike these, stores the result in a normalized registry whose reverse-dependency index makes change-impact a compile-time query.

\subsection{Type-and-Effect Systems and Algebraic Effects}
\label{sec:type-effect}

FNF-1's ``at most one effect per leaf'' condition (Section~\ref{sec:fnf1}) draws on a long programming-languages tradition. Type-and-effect systems, from Lucassen and Gifford \citep{lucassen1988} to Marino and Millstein's generic account \citep{marino2009}, extend a function's type with a summary of the effects it performs (reads, writes, allocations) over typed regions; Wadler's monadic style \citep{wadler1992} obtains the same pure/impure separation by threading effects through a monad. These systems are \emph{descriptive}: given a program, they report the effects it already has, but they do not dictate how the program should be decomposed into effectful steps.

FAOP inverts that relationship. It does not infer effects from code; the architect \emph{structures} the workflow so that each leaf carries exactly one effect symbol from a finite signature $\Sigma_E$, formalized as the generic-effect operations of Plotkin and Power \citep{plotkin2003, bauer2015}. The one-effect-per-leaf rule is therefore a design constraint on the workflow, checked against the architect's declaration rather than inferred from source. The classical effect analyzes still help: they can flag leaves that likely carry more than one effect and so tighten the declared counts, serving as architect aids rather than as the checker itself.

FAOP also does more than validate declared structure: it \emph{normalizes} that structure with the algebra of Section~\ref{sec:algebra} to remove redundancy and expose parallelism, and stores the result in a registry with reverse-dependency indexing, so the impact of changing a shared step is a compile-time query: the organization-scale concern that the declarative-compilation and contract proposals above leave unaddressed.

\section{Limitations and Future Work}
\label{sec:future-work}

\textbf{Reflective planning.} Levels 2 and 3 (Section~\ref{sec:architectural-levels}) treat the planning leaf as atomic and assume any selected or generated FAOP is normalized statically, before execution. Lifting this to a reflective calculus (FAOPs as first-class runtime values, a runtime body-evaluation primitive, and normal-form checking performed on the fly with the compiler as the runtime guardrail) is the main open problem; it interacts with the completeness result (Section~\ref{sec:formal-framework}) and with bounded completion (+B).

\textbf{Tooling and evaluation.} The compiler, registry, executor, and SkillBase (Section~\ref{sec:toolchain}) and the empirical study (Section~\ref{sec:evaluation}) are described but not yet reported. A corpus analysis of real skill libraries (violation rates per normal form, the distribution of architectural levels, and measured optimization gains) remains to be carried out.

\textbf{Transactional atomicity.} Agents that move money, send messages, or write to external systems need an all-or-nothing commit guarantee, so that a partial failure never leaves external systems half-updated. A transactional extension (+T) that confines external writes to a terminal commit and buffers internal state until then is important for production agents and is planned as near-term work.

\textbf{Runtime decomposition.} A FAOP's typed call graph makes a previously intractable question answerable: which parts of a workflow should run in which runtime, when a large agent is split across runtimes for latency, cost, data residency, or autonomy level? A cut is cheap only where the dataflow crossing it is thin, so where to cut couples directly to the independence analysis (Definition~\ref{def:independence}): choosing these partitions is a promising new research area that the FAOP structure makes tractable.

\textbf{Finding copies.} FNF-3 finds exact copies with a suffix tree, and near copies only among registered FAOPs, to keep the cost low. Better algorithms could find more at similar cost: copies that differ only by a consistent renaming of fields, as dup finds code that differs by a systematic renaming of variables~\citep{baker1995duplication}; reordered copies, by hashing summed vectors of step kinds~\citep{jiang2007deckard}; and copies with added or removed steps, by local alignment~\citep{smith1981}.

\section{Conclusion}
\label{sec:conclusion}

At scale, as AOP libraries grow to thousands, three classes of defect emerge that current agent frameworks cannot answer: failures of integrity across a shared library, failures of isolation between steps and the tools and data they touch, and the absence of an algebra to safely refactor, optimize, bound, and govern whole AOP libraries. These are structurally the problems relational data faced before Codd, and prose-and-code skills have no theory to prevent them; that is what motivates FAOP.

We presented FAOP, a theory that applies database normalization to agent skills and AOPs: a typed functional core (Section~\ref{sec:formal-framework}); four progressive normal forms for explicit atomic steps, declared dataflow, canonical reuse over an acyclic call graph, and consistent resource/guardrail declarations (Sections~\ref{sec:fnf1} to~\ref{sec:fnf4}); meaning-preserving algebraic rewrites (Section~\ref{sec:algebra}); extensions for optimization, bounded completion, and reachability checks (Section~\ref{sec:extensions}); and an architectural-level classification that identifies where a workflow invokes AI and informs governance (Section~\ref{sec:architectural-levels}). Runtime confinement and guardrail enforcement require a conforming executor and suitable policies, not declarations alone. Together these turn an unstructured library into a registry a compiler can check, summarized by each FAOP's signature $(k,E,\lambda)$. Realized as SkillBase and a separate executor, FAOP gives AOPs a governed, queryable store with schemas, constraints, and shared definitions. A new agent can then be authored for a fixed executor rather than engineered separately. In our deployment, we informally observed new-agent development fall from weeks to hours; this is an anecdotal observation, not a controlled measurement.

\vspace{0.8em}
{\small\noindent\textbf{Disclaimer}\par\smallskip
\noindent This paper was prepared in part for information purposes by the Artificial Intelligence Research group of JPMorgan Chase \& Co.\ and its affiliates (``JP Morgan''), and is not a product of the Research Department of JP Morgan. JP Morgan makes no representation and warranty whatsoever and disclaims all liability, for the completeness, accuracy or reliability of the information contained herein. This document is not intended as investment research or investment advice, or a recommendation, offer or solicitation for the purchase or sale of any security, financial instrument, financial product or service, or to be used in any way for evaluating the merits of participating in any transaction, and shall not constitute a solicitation under any jurisdiction or to any person, if such solicitation under such jurisdiction or to such person would be unlawful. \copyright\ 2026 JPMorgan Chase \& Co. All rights reserved.\par}
\vspace{0.6em}

\bibliography{references}

\appendix

\section{Normalization Examples}
\label{app:normalization-examples}

\emph{Disclaimer.} Every example in this appendix is synthetic. It is constructed solely to show how normalization works and does not describe the internal procedure of any institution. The steps, tools, guardrails, and data are deliberately simplified, and in places chosen to exercise a particular FAOP feature rather than to mirror how the task would be built in practice; any resemblance to a real workflow is incidental.

This appendix works one running example end to end: a monthly account-statement workflow. We use \emph{AOP} terminology throughout, but the same procedure applies to any source artifact whose steps are defined to accomplish a task: an SOP, an AOP, or a framework skill.

The source artifact lists the steps needed to finish the task. The Skill Architect extracts those steps (by hand or with tool assistance) into two stores: a \emph{vocabulary} of canonical names and a \emph{registry} of FAOP tuples. Extraction rewrites each unit of prose as one of the constructors of the FAOP body language (the grammar of Section~\ref{sec:composite}); this is the \emph{making steps explicit} activity of FNF-1 (Section~\ref{sec:fnf1}). A plain step becomes an action, ordered steps become a sequence $\seqop$, a conditional becomes a branch $\ifop$, a repeated step becomes a loop $\loopop$, a question put to a person becomes an interrupt $\iota$, a save point becomes a checkpoint $\chi$, and a reference to an already-registered procedure becomes a call $\callop$. Each step of the source thus yields one FAOP. The leaves are then checked for atomicity (FNF-1), after which the compiler checks typing (FNF-2), canonical reuse (FNF-3), and scope confinement (FNF-4). Extensions (+O, +B, +R) and the signature close the example.

Each registry entry is a FAOP tuple of Definition~\ref{def:faop}: the identity $\text{id}$, the name $n$ and description $d$ used for discovery, the body $b$ that performs the work (a primitive without sub-bodies for a leaf, or a composite body of sub-FAOPs for a composite), and the scope $\Sigma=\langle \mathcal{T},\mathcal{G}\rangle$. In each tuple \texttt{S\_in} names the fields the FAOP reads and \texttt{S\_out} the fields it produces; no FAOP in this example has optional outputs, so \texttt{S\_opt} is omitted. Because the record only grows along a sequence (Convention~\ref{conv:accumulation}), \texttt{S\_out} lists only the new fields the FAOP adds; the record seen by the next step is the incoming state extended with them.

\subsection{The Source AOP}

The example artifact is a monthly account-statement workflow. It lists the steps needed to finish the task, the tools the agent may use, and the guardrails it must respect.

\begin{lstlisting}[basicstyle=\ttfamily\small]
AOP: monthly account statement
Goal: produce and deliver a customer's
      monthly account statement.

Inputs:
  account_id, customer_id,
  statement_period_start,
  statement_period_end

Tools the agent may use:
  core_banking    (MCP) - read account, status
  ledger_svc      (MCP) - read transactions,
                          accrued interest
  statement_store (MCP) - store the statement
  delivery_svc    (MCP) - send to the customer
  summarizer      (LLM) - write the summary

Steps:
  1. Read the account and confirm it is active.
  2. Read the period's transactions and the
     accrued interest.
  3. Compute the ending balance.
  4. If the balance is negative, attach an
     overdraft notice; otherwise attach none.
  5. Ask the customer for a delivery channel
     and address.
  6. Write a short customer-facing summary of
     the period's activity.
    7. Store the statement once, then deliver it.
      Retry only delivery after transient
      failures, for a bounded number of attempts,
      then record delivery as failed.

Guardrails:
  - The summary must not reveal sensitive
    account data.
  - It must stay on topic, keep a neutral
    professional tone, and state only facts
    supported by the account data; it must
    not fabricate figures or advice.
  - Use only the tools listed above.
\end{lstlisting}

The steps are prose, the data each step needs is implicit, and nothing here is machine-checkable. The architect extracts this into a vocabulary and a registry (by hand, or with a model proposing a draft that the architect reviews and corrects).

\subsection{The Vocabulary}

The architect first resolves every name the AOP uses into a canonical vocabulary. Each entry carries a short description, so that two teams cannot register the same concept under different names (the \texttt{client\_id} versus \texttt{id\_client} drift of Section~\ref{sec:formal-framework}). Fields also carry a type.

\begin{lstlisting}
fields:
  account_id : string
    - unique account identifier
  customer_id : string
    - unique customer identifier
  statement_period_start : date
    - first day of the statement period
  statement_period_end : date
    - last day of the statement period
  account_record : Account
    - account object read from core_banking
  account_status : {active, closed}
    - lifecycle status of the account
  statement_transactions : Transaction[]
    - transactions posted in the period
  accrued_interest : Money
    - interest accrued in the period
  ending_balance : Money
    - closing balance for the period
  overdraft_notice : Notice
    - overdraft notice (empty if not overdrawn)
  delivery_preference : DeliveryPreference
    - customer delivery channel and address
  activity_summary : string
    - customer-facing activity summary
  statement_id : string
    - identifier of the stored statement
  delivery_confirmed : bool
    - whether delivery succeeded
  delivery_reference : string
    - reference returned by delivery_svc
\end{lstlisting}

\begin{lstlisting}
tools (op_tool unless noted):
  core_banking    - read account and status
  ledger_svc      - read transactions, interest
  statement_store - store the statement
  delivery_svc    - deliver the statement
  summarizer      - write summary  (op_LLM)

guardrails:
  no_sensitive_data_in_summary
    - omit sensitive account data
  on_topic
    - stay on the statement topic
  neutral_tone
    - keep a neutral, professional tone
  no_fabrication
    - assert only facts the data supports
\end{lstlisting}

Every external step here goes through a tool; the executor never opens a database or a socket itself. So \texttt{core\_banking}, \texttt{ledger\_svc}, \texttt{statement\_store}, and \texttt{delivery\_svc} all carry the effect label $\opTool$, and \texttt{summarizer} carries $\opLLM$. Where each tool actually runs (an in-process function, a local MCP server, or a remote MCP endpoint) is a runtime binding chosen by the executor, not part of a FAOP's type or scope.

\subsection{The Registry: from AOP to FAOPs}

Following the grammar mapping of Section~\ref{sec:fnf1}, the architect (writing the body directly or reviewing a translator model's proposal) assigns each step a FAOP with its own id, and adds the root workflow \texttt{issue\_statement.v1} that sequences them; the branch of step~4 becomes an $\ifop$ and the delivery of step~7 becomes a $\loopop$. Every FAOP has a locally authored scope: its $\mathcal{T}$ declaration is checked against the resources named by its body and called FAOPs, so \texttt{FetchAccount.v1} lists only \texttt{core\_banking} while the root lists the resources used across its body. Each FAOP also declares local guardrails. The compiler propagates them along each route, so the effective set for a descendant includes \texttt{no\_sensitive\_data\_in\_summary} from the root without merging guardrails from independent callers. \emph{Display convention:} every reference to a registered FAOP is written as $\callop(\text{id})$, whether the target is an atomic leaf or a composite sub-workflow: a call is a call. The \texttt{checkpoint} primitive is written inline and is not a call; the delivery-preference interrupt is registered as \texttt{AskDeliveryPreference.v1} and invoked like any other FAOP, so its prompt has one governed definition.

\begin{lstlisting}[basicstyle=\ttfamily\small]
id: issue_statement.v1
n:  Issue Statement
d:  Produce and deliver a monthly statement.
body:
    seq(call(FetchAccount.v1),
        call(PullTransactionsAndInterest.v1),
        call(ComputeBalance.v1),
        call(OverdraftCheck.v1),
        call(AskDeliveryPreference.v1),
        call(SummarizeActivity.v1),
        checkpoint,
        call(SaveAndDeliverStatement.v1))
S_in:  {account_id, customer_id,
        statement_period_start,
        statement_period_end}
S_out: {account_record, account_status,
        statement_transactions,
        accrued_interest, ending_balance,
        overdraft_notice, delivery_preference,
        activity_summary, statement_id,
        delivery_confirmed, delivery_reference}
T = {core_banking, ledger_svc, statement_store,
     delivery_svc, summarizer}
G = {no_sensitive_data_in_summary}
\end{lstlisting}

\begin{lstlisting}
id: FetchAccount.v1
n:  Fetch Account
d:  Read account and verify status.
body: op_tool(core_banking)
S_in:  {account_id}
S_out: {account_record, account_status}
T = {core_banking}
G = {no_sensitive_data_in_summary}
\end{lstlisting}

\begin{lstlisting}
id: PullTransactionsAndInterest.v1
n:  Pull Transactions And Interest
d:  Read transactions and accrued interest.
body:
    op_tool(ledger_svc)   # transactions
    + op_tool(ledger_svc) # accrued interest
S_in:  {account_id, statement_period_start,
        statement_period_end}
S_out: {statement_transactions,
        accrued_interest}
T = {ledger_svc}
G = {no_sensitive_data_in_summary}
\end{lstlisting}

\begin{lstlisting}
id: ComputeBalance.v1
n:  Compute Balance
d:  Calculate the ending balance.
body: op_compute(compute balance)
S_in:  {statement_transactions,
        accrued_interest}
S_out: {ending_balance}
T = {}
G = {no_sensitive_data_in_summary}
\end{lstlisting}

\begin{lstlisting}
id: OverdraftCheck.v1
n:  Overdraft Check
d:  Attach an overdraft notice if needed.
body:
    if(ending_balance < 0,
       call(AddOverdraftNotice.v1),
       call(NoNotice.v1))
S_in:  {ending_balance}
S_out: {overdraft_notice}
T = {}
G = {no_sensitive_data_in_summary}
\end{lstlisting}

\begin{lstlisting}
id: AddOverdraftNotice.v1
n:  Add Overdraft Notice
d:  Emit a customer overdraft notice.
body: op_compute(make notice)
S_in:  {ending_balance}
S_out: {overdraft_notice}
T = {}
G = {no_sensitive_data_in_summary}
\end{lstlisting}

\begin{lstlisting}
id: NoNotice.v1
n:  No Notice
d:  Emit an empty notice.
body: op_compute(empty notice)
S_in:  {}
S_out: {overdraft_notice}
T = {}
G = {no_sensitive_data_in_summary}
\end{lstlisting}

\begin{lstlisting}
id: AskDeliveryPreference.v1
n:  Ask Delivery Preference
d:  Suspend for the customer's delivery choice.
body: interrupt(customer)
      "Ask how to receive the statement (email,
      postal mail, or in-app); return it as
      delivery_preference."
S_in:  {customer_id}
S_out: {delivery_preference}
T = {}
G = {no_sensitive_data_in_summary}
\end{lstlisting}

\begin{lstlisting}
id: SummarizeActivity.v1
n:  Summarize Activity
d:  Write a customer-facing summary.
body: op_LLM(summarizer)
      "Write a short, customer-facing summary
      of the period's activity."
S_in:  {account_record, statement_transactions,
        ending_balance}
S_out: {activity_summary}
T = {summarizer}
G = {no_sensitive_data_in_summary}
\end{lstlisting}

\begin{lstlisting}[basicstyle=\ttfamily\small]
id: SaveAndDeliverStatement.v1
n:  Save And Deliver Statement
d:  Store the statement and deliver it.
body:
  seq(op_tool(statement_store),
    loop(op_tool(delivery_svc),
       until: delivery_confirmed))
S_in:  {account_id, ending_balance,
        overdraft_notice, activity_summary,
        delivery_preference}
S_out: {statement_id, delivery_confirmed,
        delivery_reference}
T = {statement_store, delivery_svc}
G = {no_sensitive_data_in_summary}
\end{lstlisting}

Once the artifact is a vocabulary plus a registry of typed FAOP tuples, its structure is fixed before any code runs: name, meaning, body, dataflow contract, and per-FAOP scope are all declared.

\paragraph{Benefits.} The gain of this first transformation is \emph{integrity across the firm's library}, the first benefit of Table~\ref{tab:benefits}:
\begin{itemize}
\item \emph{From the vocabulary:} every field, tool, and guardrail resolves to one canonical, described name shared by the whole library, so the same concept cannot drift into \texttt{client\_id} in one workflow and \texttt{id\_client} in another. This shared naming is the foundation of firm-wide consistency and integrity: workflows that mean the same thing now say it the same way.
\item \emph{From the registry $\mathcal{R}$:} each step now has one place of record (an identity, typed \texttt{S\_in}/\texttt{S\_out}, and a declared scope), so a shared step is defined once, and a later change to it can be traced to every dependent workflow instead of silently breaking them.
\end{itemize}

\subsection{Six Primitives in One Workflow}

The six primitives appear in the expanded workflow. The display opens the store-and-deliver composite to show its sequence and loop; the combined ledger action is split next.

\begin{lstlisting}[basicstyle=\ttfamily\small]
seq(                              # Sequence
  call(FetchAccount.v1),                # Action
  call(PullTransactionsAndInterest.v1), # Coarse action
  call(ComputeBalance.v1),              # Value leaf
  if(ending_balance < 0,
     call(AddOverdraftNotice.v1),
     call(NoNotice.v1)),                # Branch
  call(AskDeliveryPreference.v1),       # Interrupt
  call(SummarizeActivity.v1),           # LLM action
  checkpoint,    # Checkpoint
  op_tool(statement_store),
  loop(op_tool(delivery_svc),
       until: delivery_confirmed)  # Loop
)
\end{lstlisting}

The sequence fixes the order; the branch tests \texttt{ending\_balance < 0}; the interrupt requests delivery details; and the checkpoint persists the state. Storage runs once, before the loop that retries only delivery. +B adds the loop cap and all waiting deadlines later.

\subsection{Making Each FAOP Atomic (FNF-1)}
\label{app:fnf1}

\paragraph{Before.} FNF-1 makes one $O(\|F\|)$ pass over each FAOP and flags multi-effect leaves. Here \texttt{PullTransactionsAndInterest.v1} combines two operations in one leaf, shown by \texttt{+}. The store-and-deliver FAOP is already a composite with separate action leaves and a delivery-only loop:

\begin{lstlisting}[basicstyle=\ttfamily\small]
PullTransactionsAndInterest.v1
  body:
    op_tool(ledger_svc)   # transactions
    + op_tool(ledger_svc) # accrued interest
\end{lstlisting}

\paragraph{After.} Following Lemma~\ref{lem:leaf-split} (Section~\ref{sec:fnf1}), the architect splits the ledger leaf into two single-effect leaves. The already separate storage and delivery actions are also registered, as FNF-1 requires of every leaf other than a checkpoint. Each of the four atomic FAOPs has a scope naming exactly its own tool:

\begin{lstlisting}
id: GetTransactions.v1
n:  Get Transactions
d:  Read the period's transactions.
body: op_tool(ledger_svc)
S_in:  {account_id, statement_period_start,
        statement_period_end}
S_out: {statement_transactions}
T = {ledger_svc}
G = {no_sensitive_data_in_summary}
\end{lstlisting}

\begin{lstlisting}
id: GetInterest.v1
n:  Get Interest
d:  Read the period's accrued interest.
body: op_tool(ledger_svc)
S_in:  {account_id, statement_period_start,
        statement_period_end}
S_out: {accrued_interest}
T = {ledger_svc}
G = {no_sensitive_data_in_summary}
\end{lstlisting}

\begin{lstlisting}
id: SaveStatement.v1
n:  Save Statement
d:  Store the rendered statement.
body: op_tool(statement_store)
S_in:  {account_id, ending_balance,
        overdraft_notice, activity_summary}
S_out: {statement_id}
T = {statement_store}
G = {no_sensitive_data_in_summary}
\end{lstlisting}

\begin{lstlisting}
id: SendStatement.v1
n:  Send Statement
d:  Deliver the saved statement.
body: op_tool(delivery_svc)
S_in:  {statement_id, delivery_preference}
S_out: {delivery_confirmed, delivery_reference}
T = {delivery_svc}
G = {no_sensitive_data_in_summary}
\end{lstlisting}

The root is the only caller of the ledger FAOP and of the store-and-deliver FAOP, so the architect writes their steps in the root directly and retires both FAOPs. The root body becomes:

\begin{lstlisting}
issue_statement.v1 body:
  seq(call(FetchAccount.v1),
      call(GetTransactions.v1),
      call(GetInterest.v1),
      call(ComputeBalance.v1),
      call(OverdraftCheck.v1),
      call(AskDeliveryPreference.v1),
      call(SummarizeActivity.v1),
      checkpoint,
      call(SaveStatement.v1),
      loop(call(SendStatement.v1),
           until: delivery_confirmed))
\end{lstlisting}

The source specifies one storage operation followed by delivery retries. Its translation therefore places storage outside the loop from the start; this is a requirement of faithful translation, not a loop-hoisting optimization. FNF-1 contributes the split of the combined ledger action, not a reordering of storage and delivery.

\texttt{ComputeBalance.v1} needs no split: it reads \texttt{statement\_transactions} and \texttt{accrued\_interest}, writes \texttt{ending\_balance}, and carries the effect label $\opCompute$, a statically validated pure function of its inputs.

Some of the benefits are as follows:
\begin{itemize}
\item \emph{Observability (from $\mathcal{R}$):} each leaf carries an id, schema, and effect label, so a trace identifies exactly which unit ran.
\item \emph{Isolated testing (intrinsic):} an atomic leaf is tested against its declared \texttt{S\_in}/\texttt{S\_out} without replaying the workflow.
\item \emph{Failure localization and retry (intrinsic):} a failure attaches to one named leaf rather than to an opaque paragraph.
\item \emph{LLM bypass (from the effect label):} non-$\opLLM$ leaves such as the two \texttt{ledger\_svc} reads dispatch directly and deterministically, so only \texttt{SummarizeActivity.v1} calls the model.
\item \emph{Enabling:} atomic, effect-labeled leaves are the precondition for FNF-2 to FNF-4 and for +O parallelization.
\end{itemize}

\subsection{Checking Dataflow (FNF-2)}
\label{app:fnf2}

FNF-2 creates no FAOPs. It checks that each FAOP is called only after the fields in its \texttt{S\_in} are already present in the running record. This is where the ordering of the body matters.

\paragraph{Before.} Suppose the architect had ordered the body so that \texttt{SaveStatement.v1} runs before \texttt{SummarizeActivity.v1}:

\begin{lstlisting}
issue_statement.v1 body (rejected order):
  seq(call(FetchAccount.v1),
      call(GetTransactions.v1),
      call(GetInterest.v1),
      call(ComputeBalance.v1),
      call(OverdraftCheck.v1),
      call(AskDeliveryPreference.v1),
      call(SaveStatement.v1),       # too early
      call(SummarizeActivity.v1),
      loop(call(SendStatement.v1),
           until: delivery_confirmed))
\end{lstlisting}

\texttt{SaveStatement.v1} declares \texttt{S\_in = \{account\_id, ending\_balance, overdraft\_notice, activity\_summary\}}, but \texttt{activity\_summary} is produced only by \texttt{SummarizeActivity.v1}, which runs later. The compiler rejects the order:

\begin{lstlisting}
COMPILE ERROR in SaveStatement.v1:
  missing field activity_summary
  required S_in: {account_id, ending_balance,
                  overdraft_notice,
                  activity_summary}
  available so far: account_id, customer_id,
    statement_period_start,
    statement_period_end, account_record,
    account_status, statement_transactions,
    accrued_interest, ending_balance,
    overdraft_notice, delivery_preference
\end{lstlisting}

\paragraph{After.} Moving \texttt{SummarizeActivity.v1} before \texttt{SaveStatement.v1} makes every read available before use, and the body passes. The same rule applies to predicates: the branch condition \texttt{ending\_balance < 0} may be evaluated only after \texttt{ComputeBalance.v1} has produced \texttt{ending\_balance}.

Some of the benefits are as follows:
\begin{itemize}
\item \emph{Type safety (from the compiler):} the missing-field defect is caught before deployment instead of surfacing at runtime, where an agent would improvise the absent \texttt{activity\_summary}.
\item \emph{Auto-parallelization (enabling, from the schemas in $\mathcal{R}$):} the same typed \texttt{S\_in}/\texttt{S\_out} contracts feed the independence test that +O later uses to run leaves concurrently.
\end{itemize}

\subsection{Removing Duplication (FNF-3)}
\label{app:fnf3}

The firm's library holds more than one statement workflow.

\paragraph{Before.} Three workflows (\texttt{issue\_statement.v1}, \texttt{monthly\_report.v1}, and \texttt{on\_demand\_statement.v1}) inline the same three-step sub-workflow that builds the statement's financial facts:

\begin{lstlisting}
seq(call(GetTransactions.v1), call(GetInterest.v1),
    call(ComputeBalance.v1))
\end{lstlisting}

\paragraph{After.} The FNF-3 search finds this repeat and lists it as a candidate. Extraction is not automatic: the architect (optionally advised by a model, but making the final call) decides each pair of copies, weighing reuse, domain meaning, and ownership. This sub-workflow does the same job in all three workflows, so it is registered once, with its own scope, the union of the tools its three leaves use:

\begin{lstlisting}
id: build_statement_facts.v1
n:  Build Statement Facts
d:  Read transactions and interest,
    then compute the balance.
body:
    seq(call(GetTransactions.v1),
        call(GetInterest.v1),
        call(ComputeBalance.v1))
S_in:  {account_id, statement_period_start,
        statement_period_end}
S_out: {statement_transactions,
        accrued_interest, ending_balance}
T = {ledger_svc}
G = {no_sensitive_data_in_summary}
\end{lstlisting}

Each caller now invokes it instead of repeating its steps:

\begin{lstlisting}
issue_statement.v1 body:
  seq(call(FetchAccount.v1),
      call(build_statement_facts.v1),
      call(OverdraftCheck.v1),
      call(AskDeliveryPreference.v1),
      call(SummarizeActivity.v1),
      checkpoint,
      call(SaveStatement.v1),
      loop(call(SendStatement.v1),
           until: delivery_confirmed))
monthly_report.v1:
  seq(call(build_statement_facts.v1), ...)
on_demand_statement.v1:
  seq(call(build_statement_facts.v1), ...)
\end{lstlisting}

Some of the benefits are as follows:
\begin{itemize}
\item \emph{Reuse without duplication (from $\mathcal{R}$):} the sub-workflow is defined once and called everywhere, so the copies can no longer drift apart.
\item \emph{Single test and observability boundary (from $\mathcal{R}$):} the canonical composite has one identity, so it is tested and versioned in a single place, and its tracing and logging are configured once and standardized across the registry, so every invocation is observed the same way and every caller reuses the same governed unit.
\item \emph{Integrity across workflows (from $\mathcal{R}$):} balance logic changes in one place; every affected workflow is notified and reviewed again, and the compiler rejects structural or interface incompatibilities.
\item \emph{Impact analysis (from $\mathcal{R}$):} the reverse-dependency index returns every workflow that calls the canonical FAOP.
\end{itemize}

\subsection{Checking Scope Confinement (FNF-4)}
\label{app:fnf4}

{
Each FAOP's $\mathcal{T}(F)$ and $\mathcal{G}(F)$ are authored by the architect. FNF-4 checks the resource declaration against the recursively expanded body and compiles local guardrails along routes; it neither rewrites bodies nor infers undeclared guardrails.

\paragraph{Before.} The architect-authored $\mathcal{T}(F)$ may omit a resource named by an action leaf or called FAOP, or include one unused by the expanded body. Local guardrail declarations have not yet been propagated along each route. For a shared FAOP, merging declarations from independent callers would mix their contexts.

\paragraph{After.} The compiler derives each expected resource set bottom-up from all action leaves in each body tree and the recursively computed sets of called FAOPs, then reports missing and extra resources against the architect-authored $\mathcal{T}$. It propagates each local $\mathcal{G}$ downward separately along every route. The listing shows the declared resource sets; each composite's $\mathcal{T}$ equals the resources in its recursively expanded body.
}

\begin{lstlisting}
issue_statement.v1
  T = {core_banking, ledger_svc,
       statement_store, delivery_svc,
       summarizer}
  |- FetchAccount.v1       T = {core_banking}
  |- build_statement_facts.v1
  |    T = {ledger_svc}
  |    |- GetTransactions.v1 T = {ledger_svc}
  |    |- GetInterest.v1     T = {ledger_svc}
  |    |- ComputeBalance.v1  T = {}
  |- SummarizeActivity.v1    T = {summarizer}
  |- SaveStatement.v1        T = {statement_store}
  |- SendStatement.v1        T = {delivery_svc}
\end{lstlisting}

{
The displayed sets satisfy FNF-4-T: for example, $\mathcal{T}(\texttt{build\_statement\_facts.v1})=\{\texttt{ledger\_svc}\}$, while the root lists the union of resources used throughout its body. These are body-footprint summaries, not permissions granted to children. The omitted \texttt{OverdraftCheck.v1}, \texttt{AskDeliveryPreference.v1}, and checkpoint are tool-free.

The root locally declares $\mathcal{G}(\texttt{issue\_statement.v1})=\{\texttt{no\_sensitive\_data\_in\_summary}\}$, while \texttt{SummarizeActivity.v1} locally declares its text-generation guardrails:
}

\begin{lstlisting}
id: SummarizeActivity.v1
G = {on_topic, neutral_tone, no_fabrication}
\end{lstlisting}

{
If a new action makes \texttt{SummarizeActivity.v1} use \texttt{delivery\_svc} while its authored $\mathcal{T}$ remains unchanged, the derived resource set changes and the compiler reports the mismatch:
}

\begin{lstlisting}[basicstyle=\ttfamily]
FNF-4-T MISMATCH in SummarizeActivity.v1:
  declared T = {summarizer}
  derived  T = {delivery_svc, summarizer}
  missing  = {delivery_svc}
\end{lstlisting}

{
On the route $\pi=\texttt{issue\_statement.v1}\to\texttt{SummarizeActivity.v1}$, the compiled guardrail set is
\[
\begin{aligned}
\mathcal{G}^{*}(\pi)=\{&\texttt{no\_sensitive\_data\_in\_summary},\\
&\texttt{on\_topic},\ \texttt{neutral\_tone},\\
&\texttt{no\_fabrication}\}.
\end{aligned}
\]
A different caller of a shared FAOP receives a separately computed set.
}

Some of the benefits are as follows:
\begin{itemize}
\item \emph{Scope isolation:} FNF-4 checks declarations. A conforming executor mediating all resource access then confines \texttt{SummarizeActivity.v1} to its approved $\mathcal{T}=\{\texttt{summarizer}\}$, denying access to \texttt{delivery\_svc}; the declaration alone does not enforce this.
\item \emph{Compiler and runtime policy:} resource mismatches are reported before deployment, and guardrail declarations remain separate per route. Their semantic enforcement requires suitable runtime checks; propagation alone does not prove that a generated summary obeys them.
\item \emph{From the guardrails:} the compiled set on the displayed route contains both the root's \texttt{no\_sensitive\_data\_in\_summary} and \texttt{SummarizeActivity.v1}'s local text-generation guardrails.
\end{itemize}

\subsection{Extensions on the Same Example}
\label{app:extensions}

Extensions annotate existing FAOPs; +B also adds a bounded exhaust-handler FAOP.

\paragraph{+O Optimized.} The body of \texttt{build\_statement\_facts.v1} sequences \texttt{GetTransactions.v1}, \texttt{GetInterest.v1}, and \texttt{ComputeBalance.v1}. Its first two leaves are independent (\texttt{GetTransactions.v1} writes only \texttt{statement\_transactions}, \texttt{GetInterest.v1} writes only \texttt{accrued\_interest}, and neither reads a field the other writes), so +O runs them in parallel by the commutativity law (L3):

\begin{lstlisting}
build_statement_facts.v1 body:
  seq(par(call(GetTransactions.v1),
          call(GetInterest.v1)),
      call(ComputeBalance.v1))
\end{lstlisting}

The two ledger calls also require declared read-only access under a common snapshot or fixed statement-period version. FNF-2 schemas alone do not establish this external-resource condition. Under those contracts they may run concurrently, and \texttt{ComputeBalance.v1} still receives both outputs before it runs; otherwise the compiler keeps their order.

+O also fuses independent interrupts. The example asks the customer only once, so nothing fuses here; but if \texttt{AskDeliveryPreference.v1} were split into two questions (one for the channel, one for the postal address) with disjoint answer keys, +O would merge them into a single customer prompt, turning two round-trips into one.

\paragraph{+B Bounded.} The delivery loop receives a maximum iteration count and an \emph{exhaust handler}, which reports unsuccessful delivery if the cap is reached:

\begin{lstlisting}[basicstyle=\ttfamily\small]
id: MarkDeliveryFailed.v1
n:  Mark Delivery Failed
d:  Record exhausted delivery attempts.
body: op_compute(mark failed)
S_in:  {statement_id}
S_out: {delivery_confirmed, delivery_reference}
T = {}
G = {no_sensitive_data_in_summary}

loop(call(SendStatement.v1),
     until: delivery_confirmed,
     max_iterations: 3,
     on_exhaust: call(MarkDeliveryFailed.v1))

Latency limits for this synthetic example:
  FetchAccount.v1:       2s
  GetTransactions.v1:    2s
  GetInterest.v1:        2s
  SummarizeActivity.v1:  4s
  SaveStatement.v1:      2s
  SendStatement.v1:      5s
  every op_compute leaf: 0.1s
  checkpoint:           1s
  interrupt response deadline: 24h
  control/dispatch/predicate work:
    at most 0.01s per visit

Token limits:
  SummarizeActivity.v1: 800 total billable tokens
  all other work:      0 tokens

Timeout policy: bounded failure reporting;
no implicit retry outside the capped loop.
\end{lstlisting}

These illustrative limits must cover the whole invocation, including cancellation and failure reporting, and be enforceable by the executor and the model/service providers. The three delivery attempts contribute at most $15\mathrm{s}$, to which the compiler adds loop/control overhead and the $0.1\mathrm{s}$ exhaust handler. The whole-workflow bound also includes the customer deadline, the other leaves, checkpoint persistence, and frame/dispatch overhead. Omitting any of these would not justify +B.

\paragraph{+R Reachable.} The obligation of part (i) is the reachability of both overdraft branches through the ledger reads and the balance computation, not merely the satisfiability of \texttt{ending\_balance < 0}. For this synthetic example, the architect supplies entry and transition contracts admitting both negative and nonnegative balances, and bounded model checking must find a witness run for each arm. The +R annotation is conditional on those witnesses; without them, it is withheld. Part (ii) needs no contracts: every pure step in the example writes a field that a later step reads or the workflow returns; for example, \texttt{ending\_balance} is read by the overdraft branch and by \texttt{SaveStatement.v1}. The delivery loop's stopping limit is not part of +R: it belongs to +B.

\subsection{Algebraic Laws as Tutorial Rewrites}
\label{app:laws}

The algebra of Section~\ref{sec:algebra} explains why these rewrites are not merely engineering heuristics. Every FAOP named below is already in the registry built above.

\paragraph{L1 Identity: empty-step elimination.} The law $\seqop(\nopsym,f)=f$ permits insertion or removal of the empty do-nothing step, which changes neither the record nor any external state. Logging, metrics, and checkpoints have observable effects even when their \texttt{S\_out} is empty, so L1 does not permit removing them.

\paragraph{L2 Associativity: safe extraction.} Let $T$, $I$, and $B$ abbreviate \texttt{GetTransactions.v1}, \texttt{GetInterest.v1}, and \texttt{ComputeBalance.v1}. Associativity makes the grouping of a sequential composition irrelevant:
\[
\seqop(\seqop(T,I),B)=\seqop(T,\seqop(I,B)),
\]
so the three-step sub-workflow $\seqop(T,I,B)$ is well-defined independently of how it is bracketed. This is what licenses treating the whole sub-workflow as a single unit and registering it as the composite \texttt{build\_statement\_facts.v1} under FNF-3, after which it is invoked as one $\callop$ step.

\paragraph{L3 Commutativity: auto-parallelization.} Since \texttt{GetTransactions.v1} writes \texttt{statement\_transactions}, \texttt{GetInterest.v1} writes \texttt{accrued\_interest}, and neither reads the other's output,
\[
\seqop(T,I)=\parop(T,I).
\]
With the declared read-only ledger snapshot contract, the executor may run them concurrently without changing the logical result.

\paragraph{L4 Branch distribution: shared-step factoring.} A shared continuation can sit after the branch or be distributed into both arms:
\[
\seqop(\ifop(p,N,Z),R)=\ifop(p,\seqop(N,R),\seqop(Z,R)).
\]
For the example, $p$ is \texttt{ending\_balance < 0}, $N$ is \texttt{AddOverdraftNotice.v1}, $Z$ is \texttt{NoNotice.v1}, and $R$ can be \texttt{SaveStatement.v1} in a surrounding state where the other inputs of \texttt{SaveStatement.v1} are already available. The two forms save a statement whose \texttt{overdraft\_notice} comes from the chosen branch; read right to left, the law lets a step that ends both arms be written once, after the branch.

\paragraph{L5 Fold/unfold: canonical reuse.} For $C=\mathcal{R}[\texttt{build\_statement\_facts.v1}]$, L5 equates its call with $\operatorname{frame}_{C}(b(C))$, not with an unframed copy. Once L2 groups $\seqop(T,I,B)$, extraction declares the inputs and every output used by the caller. Unfolding keeps that same frame: it passes only $\Sin(C)$ and returns only $\Sout(C)$, leaving intermediate fields local. With this boundary preserved, folding and unfolding compute the same result.

\subsection{The Signature of the Example}
\label{app:signature}

With the +B and read-only snapshot contracts enforced and the stated +R obligations discharged, and any remaining FNF-3 candidate pairs, such as \texttt{GetTransactions.v1} and \texttt{GetInterest.v1} should their similarity reach the threshold, declared distinct, the synthetic workflow has signature
\[
(k,E,\lambda)=(4,\{+O,+B,+R\},1).
\]
Here $k=4$ because the workflow satisfies FNF-1 through FNF-4: leaves are atomic, dataflow is typed, shared sub-workflows are canonical, and FNF-4 matches resource declarations to body footprints while compiling guardrails per route. The extension set records checked optimizations (+O), bounded completion or failure with time/cost limits (+B), and reachability checks (+R). The architectural level is $\lambda=1$ because the graph and routing are fixed and computable, while one leaf, \texttt{SummarizeActivity.v1}, invokes an LLM as a tool. If that leaf is removed, the same workflow becomes Level 0.

\subsection{Does This Workflow Need AI?}
\label{app:need-ai}

Read as prose, the source AOP looks like an AI task: it is handed to an agent, and the reflex is to run the whole thing through a language model. Normalization shows that reflex to be wrong. \textbf{Exactly one leaf carries the effect label $\opLLM$}: \texttt{SummarizeActivity.v1}. Every other leaf is a tool call, a value computation, an interrupt, or a checkpoint, and the only routing predicate, \texttt{ending\_balance < 0}, is an ordinary boolean over the state. This is why the level is $\lambda=1$ rather than something higher: the graph and routing are fixed and computable, and the model is used as a tool at a single point.

\keybox{The lesson runs against the habit of putting a language model in charge of everything. The statement is fetched, computed, stored, and delivered by deterministic steps that need no model, cost no tokens, and are fully reproducible; only the customer-facing prose genuinely calls for generation. Making that boundary explicit lets a firm reserve model spend, latency, and model-risk review for the one leaf that needs them and run the rest as a plain, auditable program. Replace that leaf with a templated summary and the workflow drops to $\lambda=0$: no AI at all.}

\subsection{Adding the Next AOP}

One normalized workflow is only the start. Suppose a second artifact arrives, an \texttt{account\_closure} AOP that closes an account and issues a final statement. It is normalized against the \emph{same} two stores: its steps resolve to the existing vocabulary, so it reads \texttt{account\_id} and \texttt{ending\_balance} rather than inventing \texttt{acct\_no} or \texttt{closing\_bal}, and FNF-3 makes it a caller of the registered \texttt{FetchAccount.v1} and \texttt{build\_statement\_facts.v1} rather than a fresh copy:

\begin{lstlisting}
account_closure.v1 body:
  seq(call(FetchAccount.v1),
      call(build_statement_facts.v1),
      call(IssueFinalStatement.v1),
      call(CloseAccount.v1))
\end{lstlisting}

The new body also repeats the two steps \texttt{call(FetchAccount.v1)} and \texttt{call(build\_statement\_facts.v1)} of \texttt{issue\_statement.v1}, so the next FNF-3 round reports this pair, and the architect factors it or declares it distinct.
Repeating this across the firm's AOPs creates one registry over one vocabulary. Each shared sub-workflow has one canonical definition, and the reverse-dependency index identifies every affected workflow when it changes. All affected workflows and their owners are notified and must review the change again before adopting it. The compiler catches structural and interface errors; business-meaning changes still require tests and review.

\subsection{Exporting}
\label{app:exporting}

A normalized FAOP is exported in either of the two forms of Section~\ref{sec:toolchain}:

\begin{itemize}
\item \emph{Blueprint}: the full form. A FAOP-aware executor materializes \texttt{issue\_statement.v1} into a runnable graph, binding \texttt{core\_banking}, \texttt{ledger\_svc}, \texttt{statement\_store}, and \texttt{delivery\_svc}, registering the interrupt and checkpoint handlers, and enforcing each leaf's scope at run time. This form carries every guarantee, including runtime scope confinement and effect-based LLM bypass.
\item \emph{Adapted AOP}: the degraded form. An adapter converts the same blueprint into an AOP in a target executor's own format, so an existing framework can run it directly. It \emph{keeps} the design-time structure (atomic typed steps, canonical reuse, and a clean description for discovery) and keeps the runtime guarantees, chiefly scope confinement and LLM bypass, to the extent that the target executor can enforce them.
\end{itemize}

The adapted AOP is \emph{generated from the registry}, not hand-maintained. When a shared FAOP changes, each affected workflow is notified and reviewed again before its updated blueprint or adapted export is deployed. Regeneration from the approved canonical source prevents copy drift; it does not replace review of changed business meaning. The adapter remains a backward-compatible path for existing runtimes, retaining the checked structure while runtime guarantees depend on what the target executor enforces.

\end{document}